\documentclass[english]{article}
\usepackage[a4paper,margin=2cm]{geometry}
\usepackage{setspace}
\usepackage{amsmath}
\usepackage{amssymb}
\usepackage{babel}

\usepackage{mathptmx}
\usepackage{amsmath}
\usepackage{graphicx}
\usepackage{color}

\usepackage{amsmath} 
\usepackage{amsthm} 
\usepackage{epsfig}
\usepackage{graphicx} 

\usepackage{caption}
\usepackage{subcaption}

\newcommand{\rb}{\rangle}
\newcommand{\lb}{\langle}

\renewcommand{\Pr}{\mathbf{Pr}}
\newcommand{\R}{\mathbb{R}}
\newcommand{\E}{\mathbb{E}}
\newcommand{\N}{\mathcal{N}}

\newtheorem{theorem}{Theorem}[section]
\newtheorem{lemma}{Lemma}[section]

\newtheorem{proposition}{Proposition}[section]

\begin{document}

\title{Deep Haar Scattering Networks}

\author{{
\sc Xiuyuan Cheng}$^*$,\\[2pt]
Applied Mathematics Program, Yale University, CT, USA\\
xiuyuan.cheng@yale.edu\\[2pt]
{\sc Xu Chen}\\[2pt]
Electrical Engineering Department,
Princeton University, NJ, USA\\
{ xuchen@princeton.edu }\\[2pt]
{\sc and}\\[2pt]
{\sc St\'ephane Mallat} \\[2pt]
Computer Science Department, \'Ecole Normale Sup\'erieure, Paris, France\\
{stephane.mallat@ens.fr}}

\maketitle

\begin{abstract}
{
An orthogonal Haar scattering transform is a deep network,
computed with a hierarchy of
additions, subtractions and absolute values, over pairs of coefficients.
It provides a simple mathematical model for unsupervised deep network learning.
It implements non-linear contractions, which are optimized for classification,
with an unsupervised pair matching algorithm, of polynomial complexity.
A structured Haar scattering over graph data computes permutation invariant
representations of groups of connected points in the graph. If the graph connectivity
is unknown, unsupervised Haar pair learning can provide a consistent estimation of
connected dyadic groups of points. Classification results are given on image
data bases, defined on regular grids or graphs, with a connectivity which may
be known or unknown.
}
{ 
deep learning, neural network, scattering transform, Haar wavelet, classification, 
images, graphs
}
\\
2000 Math Subject Classification: 68Q32, 68T45, 68Q25, 68T05
\end{abstract}


\section{ Introduction}

Deep neural networks appear to provide scalable learning architectures for high-dimensional learning, with impressive results on many different type of data and signals \cite{bengio2013representation}. Despite their efficiency, there is still little understanding on the properties of these architectures.  Deep neural networks alternate pointwise  linear operators, whose coefficients are optimized with training examples, with pointwise non-linearities. To obtain good classification results, strong constraints are imposed on the network architecture on the support of these linear operators \cite{LeCun}. These constraints are usually derived from an experimental trial and error processes.

Section \ref{sec:2} introduces a simple deep Haar scattering architecture, which only computes the sum of pairs of coefficients, and the absolute value of their difference. The architecture preserves some important properties of deep networks, while reducing the computational complexity and simplifying their mathematical analysis. Through this architecture, we shall address major questions concerning invariance properties, learning complexity, consistency, and the specialization of such architectures.

Convolution networks are particular classes of deep networks, which compute translation invariant descriptors of signals defined over uniform grids \cite{LeCun, sermanet2013overfeat}. Scattering networks were introduced as convolution networks computed with iterated wavelet transforms, to obtain invariants which are stable to deformations \cite{mallat2012group}. 
With appropriate architecture constraints on Haar scattering networks, 
Section \ref{sec:3} defines locally displacement invariant representations of signals defined on general graphs. In social, sensor or transportation 
networks, high dimensional data vectors are supported on 
a graph \cite{shuman2013emerging}. 
In most cases, propagation phenomena require to define translation invariant representations for classification.  We show that an appropriate configuration of an orthogonal Haar scattering defines such a translation invariant representation on a graph. It is computed with a product of Haar wavelet transforms on the graph, and is thus closely related to non-orthogonal translation invariant scattering transforms \cite{mallat2012group}.

The connectivity of graph data is often unknown. In social or financial networks, we typically have information on individual agents, without knowing the interactions and hence
connectivity between agents. Building invariant representations on such graphs requires to estimate the graph connectivity. Such information can be inferred from unlabeled data, by  analyzing the joint variability of signals defined on the unknown graph. This paper studies unsupervised learning strategies, which optimize deep network configurations, without class label information.

Most deep neural networks  are fighting the curse of dimensionality by reducing the variance of the input data with contractive non-linearities \cite{rifai2011contractive, bengio2013representation}. The danger of such contractions is to nearly collapse together vectors which belong to different classes. 
Learning must
preserve discriminability despite this variance reduction resulting from contractions.  
Hierarchical unsupervised
architectures have been shown to provide efficient learning 
strategies \cite{anselmi2013unsupervised}. 
We show that unsupervised learning can optimize
an average discriminability by computing sparse features.
Sparse unsupervised learning, which is usually NP hard, is reduced to a pair matching problem for Haar scattering. It can thus be computed with a polynomial complexity algorithm. For Haar scattering on graphs, it
recovers a hierarchical connectivity of groups of vertices.
Under appropriate assumptions, we prove that pairing problems avoid the curse of dimensionality. It can recover an exact connectivity with arbitrary high probability
if the training size grows almost linearly with the signal dimension, 
as opposed to exponentially. 

Haar scattering classification architectures are numerically tested
over image databases defined on uniform grids or irregular graphs, 
whose geometries are either known or estimated by unsupervised learning. 
Results are compared with state of the art unsupervised and supervised
classification algorithms, applied to the same data, with 
known or unknown geometry.
All computations can be reproduced with a software available at {\it{www.di.ens.fr/data/scattering/haar}}.


\section{Free Orthogonal Haar Scattering}\label{sec:2}

\subsection{Orthogonal Haar filter contractions}
\label{freesec}

We progressively introduce orthogonal Haar scattering 
by specializing a general deep neural network. We 
explain the architecture constraints and the resulting
contractive properties.

The input network is a positive
$d$-dimensional signal $x \in (\R^+)^d$,  which we write $S_0 x  = x$. 
We denote by $S_j x$ the network layer at the depth $j$.  A deep neural network computes the next network layer by applying a linear operator $H_j$ to $S_{j} x$ followed by a non-linear operator. 
Particular deep network architectures impose that $H_j$ preserves distances, up to a constant normalization factor $\lambda$ \cite{ngiam2010tiled}:
\[
\|H_j y - H_j y'\| = \lambda\, \|y -y'\|~.
\]
The network is contractive if it applies a pointwise contraction $\rho$ to each value of the output vector $H_j S_j x$. 
This means that for any $(a,b) \in \R^2$
$|\rho(a) - \rho(b)| \leq |a-b|$. Rectifications and sigmoids are examples
of such contractions. We use an absolute value $\rho(a)=|a|$ because
it preserves the amplitude, and it yields a permutation invariance which
will be studied. For any vector $y = (y(n))_n$, the pointwise absolute value is written $|y| = (|y(n)|)_n$. The next network layer is thus:
\begin{equation}
\label{permasnbsdfu}
S_{j+1} x = | H_j S_j x|~. 
\end{equation}
This transform is iterated up to a maximum depth $J \leq \log_2(d)$ to compute the network output $S_J x$.

We shall further impose that each layer $S_j x$ has the same dimension as $x$, 
and hence that $H_j$ is an orthogonal operator in $\R^d$, up to the  scaling factor $\lambda$.
Geometrically, $S_{j+1} x$ is thus obtained by rotating $S_j x$ with $H_j$, 
and by contracting each of its coordinate with the absolute
value. The geometry of this contraction is thus defined by
the choice of the operator $H_j$ which adjusts
the one-dimensional directions along which the contraction is performed. 

An orthogonal Haar scattering is implemented with an orthogonal Haar filter $H_j$. 
The vector $H_j y$ regroups the coefficients of $y \in \R^d$ into $d/2$ pairs
and computes their sums and differences. The rotation $H_j$ is thus factorized
into $d/2$ rotations by $\pi/4$ in $\R^2$, and multiplications by $2^{1/2}$.
The transformation of each coordinate pair $(\alpha,\beta) \in \R^2$ is:
\[
(\alpha,\beta) ~\longrightarrow
(\alpha+\beta\,,\,\alpha-\beta\,).
\]
The operator $|H_j|$
applies an absolute value to each output
coordinate, which 
has no effect on $\alpha+\beta$ if $\alpha \geq 0$ and $\beta \geq 0$,
but it removes the sign of their difference:
\begin{equation}
\label{permasn}
(\alpha,\beta) ~\longrightarrow
~(\alpha+\beta\,,\,|\alpha-\beta|)~. 
\end{equation}
Observe that this non-linear operator
defines a permutation invariant
representation of $(\alpha,\beta)$. Indeed,
the output values are not modified by  a permutation of $\alpha$ and $\beta$, and 
the two values of $\alpha$, $\beta$ are recovered without order, by 
\begin{equation}\label{eq:recover}
\max(\alpha,\beta) = \frac 1 2 \big(\alpha+\beta+|\alpha-\beta| \big)~~\mbox{and}~~
\min(\alpha,\beta) = \frac 1 2 \big(\alpha+\beta-|\alpha-\beta| \big)~.
\end{equation}
The operator $|H_j|$ can thus also be interpreted as a calculation of $d/2$ permutation
invariant representations of pairs of coefficients.

\begin{figure}
\centering
\includegraphics[width=0.65\linewidth]{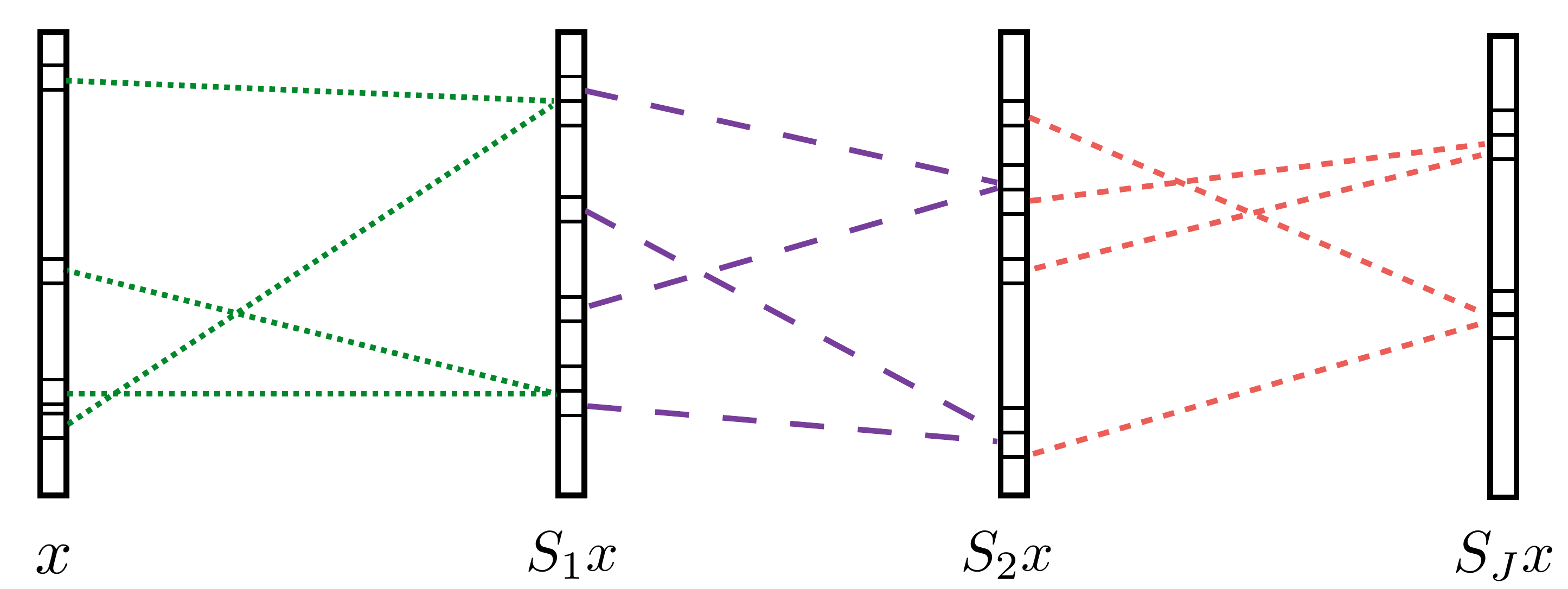}
\caption{ 
\small  A free Haar scattering network computes a layer $S_{j+1}x$ 
by pairing the coefficients of the previous layer $S_j x$, and storing the
sum of coefficients in each pair and the amplitude of their difference.
}
\label{fig:1}
\end{figure}

Applying $|H_j|$ to $S_j x$ computes the next layer $S_{j+1} x = |H_j S_j x|$, obtained by regrouping the coefficients of $S_j x \in \R^d$ into $d/2$ pairs of indices written
$\pi_j = \{\pi_j(2n),\pi_j(2n+1)\}_{0 \le n < d/2}$:
\begin{equation}
\label{eqn1}
S_{j+1} x (2n) = S_j x(\pi_j(2n)) + S_j x(\pi_j(2n+1))~,
\end{equation}
\begin{equation}
\label{eqn2}
S_{j+1} x(2n+1) = |S_j x(\pi_j(2n)) - S_j x(\pi_j(2n+1))|~.
\end{equation}
The pairing $\pi_j$  specifies which index 
$\pi_j (2n+1)$ is paired
with $\pi_j(2n)$, but the ordering index $n$ is not important. It specifies
the storing position in $S_{j+1} x$ of the transformed values.
For classification applications, $\pi_j$ will be optimized with training examples. 
This deep network computation is illustrated in Figure \ref{fig:1}. The network output $S_J x$ is calculated with $J d/2$ additions, subtractions and absolute values.
Each coefficient of $S_J x$ is calculated by cascading $J$ permutation invariant
operators over pairs, and thus defines an invariant over a group of $2^J$ coefficients.
The network depth $J$ thus corresponds to an invariance scale $2^J$.

Since the network is computed by iterating orthogonal linear operators, up to a normalization, and a contractive absolute value, the following theorem proves that it defines a contractive transform, which preserves the norm, up to a normalization. It also proves that an orthogonal Haar scattering transform $S_J x$ is obtained by
applying an orthogonal matrix to $x$, which depends upon $x$ and $J$. 

\begin{theorem}\label{prop:norm-preserve}
For any $J \geq 0$, and any $(x,x') \in \R^{2d}$
\begin{equation}
\label{nsdf8}
\|S_J x - S_J x' \| \leq 2^{J/2} \|x - x' \|~.
\end{equation}
Moreover $S_J x = 2^{J/2} \, M_{x,J}\, x$ where $M_{x,J}$ 
is an orthogonal matrix which depends on $x$ and $J$, and
\begin{equation}
\label{nsdf0}
\|S_J x\| = 2^{J/2} \|x \|~.
\end{equation}
\end{theorem}

\begin{proof}
Since $S_{j+1} x = |H_j S_j x|$ where 
$H_j$ is an orthogonal operator multiplied by $2^{1/2}$, 
\[
\|S_{j+1} x - S_{j+1} x' \| \leq \|H_j S_j x - H_j S_j x' \| = 2^{1/2} \|S_j x - S_j x'\|.
\]
Since $S_0 x = x$, equation (\ref{nsdf8}) is verified by induction on $j$.
We can also rewrite
\[
S_{j+1} x = |H_j S_j x| = E_{j,x} H_j x,
\]
where $E_{j,x}$ is a diagonal matrix where the diagonal entries are $\pm 1$, 
with a sign which depend on $S_j x$. Since $2^{-1/2} H_j$ is orthogonal,  
$2^{-1/2} E_{j,x} H_j$ is also orthogonal
so $M_{x,J} = 2^{-J/2} \prod_{j=1}^J E_{j,x}H_j$ is orthogonal, and depends on $x$ and $J$. It results that $\|S_J x \| = 2^{J/2}\,\|x\|$.
\end{proof}

\subsection{Complete representation with bagging}

A single Haar scattering transform looses information since it applies orthogonal operators followed by an absolute value which looses the information of the sign. However, the following theorem proves that $x$ can be recovered from 
$2^J$ distinct orthogonal Haar scattering transforms, computed
with different pairings $\pi_j$ at each layer.

\begin{theorem}
\label{thm:multi-layer-recovery2}
There exist $2^J$ different orthogonal Haar scattering transforms such that 
almost all $x \in {\mathbb R^d}$ can be reconstructed from the coefficients of 
these $2^J$ transforms.
\end{theorem}

This theorem is proved by observing that a Haar scattering transform is computed with permutation invariants operators over pairs. Inverting these operators allows to recover values of signal pairs but not their locations. However, recombining these values on enough overlapping sets allows one to recover their locations and hence the original signal $x$. This is proved on the following lemma applied to {\it interlaced pairings}.
We say that two pairings  $\pi^{0} = \{ \pi^{0}(2n),  \pi^{0}(2n+1) \}_{0 \le n < d/2}$  and $\pi^{1} = \{ \pi^{1}(2n),  \pi^{1}(2n+1) \}_{0 \le n < d/2}$  are interlaced if there exists no strict subset $\Omega$  of $\{1,...,d\}$ such that  $\pi^0$ and $\pi^{1}$ are pairing elements within $\Omega$. The following lemma shows that a single-layer scattering operator is invertible with two interlaced pairings.

\begin{lemma}\label{lemma:interlacing}
If two pairings $\pi^0$ and $\pi^1$ of $\{1, ..., d\}$ are interlaced then any $x \in \R^d$ whose coordinates have more than $2$ different values can be recovered from 
the values of $S_1 x$ computed with $\pi^0$ and the values of $S_1 x$ computed with $\pi^1$.
\end{lemma}
\begin{proof}
Let us consider a triplet ${n_1, n_2, n_3}$ where
$(n_1, n_2)$  is a pair in $\pi^0$ and
$(n_1, n_3)$  is a pair in $\pi^1$. From $S_1 x$ computed with $\pi^0$ we get
\[
x(n_1) + x(n_2), \quad |x(n_1) - x(n_2)|
\]
and we saw in (\ref{eq:recover}) that  it 
determines the values of  $\{ x(n_1), x(n_2)\}$ up to a permutations. 
Similarly, $\{ x(n_1), x(n_3)\}$ are determined up to a permutation
by $S_1 x$ computed with $\pi^1$.
Then unless $x(n_1) \neq x(n_2)$ and $x(n_2) = x(n_3)$ the three values $x(n_1), x(n_2), x(n_3)$ are recovered. The interlacing condition implies that  $\pi^1$ pairs $n_2$ to an index $n_4$ which can not be $n_3$ or $n_1$. Thus, the four values of $x(n_1), x(n_2), x(n_3), x(x_4)$ are specified unless $x(n_4) = x(n_1) \neq x(n_2) = x(n_3)$.  This interlacing argument can be used to extend  to $\{1,\dots, d\}$ the set of all indices $n_i$ for which $x(n_i)$ is specified, unless $x$ takes only two values.
\end{proof}

\begin{proof}[Proof of Theorem \ref{thm:multi-layer-recovery2}]
Suppose that  the $2^J$ Haar scatterings are associated to the $J$ hierarchical pairings $(\pi_1^{\epsilon_1}, ..., \pi_J^{\epsilon_J} )$ where $\epsilon_j \in \{0,1\}$, where for each $j$, $\pi_j^{0}$ and $\pi_j^{1}$ are two interlaced pairings of $d$ elements.  The sequence $(\epsilon_1, ..., \epsilon_J)$  is a binary vector taking $2^J$ different values. 

The constraint on the signal $x$ is that each of the intermediate scattering coefficients takes more than $2$ distinct values, which holds for $x \in \R^d$ except for a union of hyperplanes which has zero measure. Thus for almost every $x \in \R^d$, the theorem follows from applying Lemma \ref{lemma:interlacing} recursively to the $j$-th level scattering coefficients for $ J-1 \ge j \ge 0 $.
\end{proof}

Lemma \ref{lemma:interlacing} proves that only two pairings is sufficient to  invert one Haar scattering layer. The argument proving that $2^J$ pairings are sufficient to invert $J$ layers is quite brute-force. It is conjectured that the number of pairings 
needed to obtain a complete representation for almost all $x \in \R^d$ does not need
to grow exponentially in $J$ but rather linearly.
Theorem \ref{thm:multi-layer-recovery2}  suggests to define a signal 
representations by aggregating different Haar orthogonal scattering transforms. We shall see that this bagging strategy is indeed improving classification accuracy.

\subsection{Sparse unsupervised learning with adaptive contractions}\label{subsec:unsup-learn}

A free orthogonal Haar scattering transform of depth $J$ is computed with a pairing at each of the $J$ network layers, which may be chosen freely.  
We now explain how to optimize these pairings from $N$ unlabeled examples $\{x_i \}_{1 \leq i \leq N}$. As previously explained, an orthogonal Haar scattering is strongly contractive. Each linear Haar operator rotates the signal space,  and the absolute value suppresses the sign of each difference, and hence projects coefficients over a smaller domain.  Optimizing the network thus amounts to find the best directions along which to perform the space compression. 

Contractions reduce the space volume and hence the variance of scattering vectors but it may also collapse together examples which belong to different classes. To maximize the ``average discriminability'' among signal examples, we shall thus maximize the variance of the scattering transform over the training set. 
Following \cite{mallat2013deepscat}, we show that it yields
a representation whose coefficients are sparsely excited. 

The network layers are
optimized with a greedy layerwise strategy similar to many deep unsupervised learning algorithms \cite{hinton2006fast, bengio2013representation},  which consists 
in optimizing the network parameters layer per layer, as the depth  $j$ increases. 
Let us suppose that Haar scattering operators $H_\ell$ are computed for $1 \leq \ell < j$. One can thus compute $S_j x$ for any $x \in \R^d$. We now explain how to optimize $H_j$ to maximize the variance of the next layer $S_{j+1} x$.  The non-normalized empirical variance of $S_j$ over the training set $\{x_i\}_i$ is
\[
\sigma^2 (S_{j} x) = \sum_i \| S_{j} x_i \|^2 - \Big\| \sum_i S_{j} x_i \Big\|^2~.
\]
The following proposition, adapted from  \cite{mallat2013deepscat},
proves that the scattering variance decreases as the depth increases, up to a factor $2$. It gives a condition on $H_j$ to maximize the variance of the next layer. 

\begin{proposition}
For any $j \geq 0$ and $x \in \R^d$, 
$\sigma^2 ( 2^{-(j+1)/2} S_{j+1} x) \leq \sigma^2( 2^{-j/2}  S_j x)$. 
Maximizing $\sigma^2 (S_{j+1} x)$ given $S_j x$ is equivalent to finding $H_j$ which minimizes
\begin{equation}
\label{l1energy0}
\Big\| \sum_i H_j S_j x_i \Big\|^2 = 
\sum_n \Big( \sum_i | H_j S_j x_i(n)| \Big)^2 ~.
\end{equation}
\end{proposition}

\begin{proof}
Since  $S_{j+1} x= | H_j S_j x|$ and $\| H_j S_j x \| = 2^{1/2} \|S_j x \|$, we have
\begin{eqnarray*}
\sigma^2 (S_{j+1} x) 
&=& \sum_i \| S_{j + 1} x_i \|^2 - \Big\| \sum_i S_{j+1} x_i \Big\|^2\\
&=& 2 \sum_{i=1}^N \| S_{j} x_i \|^2 - \Big\| \sum_{i=1}^N |H_j S_{j} x_i| \Big\|^2.
\end{eqnarray*}
Optimizing $\sigma^2 (S_{j+1} x)$ is thus equivalent to minimizing (\ref{l1energy0}). Moreover,
\begin{eqnarray*}
\sigma^2 (S_{j+1} x) 
&=& 2 \sum_{i=1}^N \| S_{j} x_i \|^2 
    - \Big\| H_j \sum_{i=1}^N  S_{j} x_i \Big\|^2
   + \Big\|  \sum_{i=1}^N  H_j S_{j} x_i \Big\|^2 
   - \Big\| \sum_{i=1}^N |H_j S_{j} x_i| \Big\|^2~\\
& = &2 \sum_{i=1}^N \| S_{j} x_i \|^2 
     -  2 \|\sum_{i=1}^N  S_{j} x_i \|^2
    + \left(  \Big\|  \sum_{i=1}^N  H_j S_{j} x_i \Big\|^2 
   - \Big\| \sum_{i=1}^N |H_j S_{j} x_i| \Big\|^2 \right)~\\
&\le & 2 \sum_{i=1}^N \| S_{j} x_i \|^2 
     -  2 \|\sum_{i=1}^N  S_{j} x_i \|^2 = 2 \sigma^2(S_j x),
\end{eqnarray*}
which proves the first claim of the proposition.
\end{proof}

This propocsition relies on the energy conservation $\|H_j y \| = 2^{1/2} \|y \|$.  Because of the contraction of the absolute value, it proves that the variance of the normalized scattering $2^{-j/2} S_j x$ decreases as $j$ increases.  Moreover the maximization
of $\sigma^2(S_{j+1} x)$ amounts to minimize a mixed $\bf l^1$ and $\bf l^2$ norm
on $H_j S_j x_i (n)$, where the sparsity $\bf l^1$ norm is along the realization index $i$ where as the
$\bf l^2$ norm is along the feature index $n$ of the scattering vector. 

Minimizing the first $\bf l^1$ norm for $n$ fixed tends
to produce a coefficient
indexed by $n$ which is sparsely excited across the examples indexed by $i$. 
It implies that this feature is discriminative among all examples. On the contrary,
the $\bf l^2$ norm along the index $n$ has a tendency to produce $\bf l^1$ sparsity norms
which have a uniformly small amplitude. The resulting ``features'' indexed by $n$ are thus uniformly sparse. 

Because $H_j$ preserves the norm, the total energy of coefficients is conserved:
\[
\sum_n \sum_i |H_j S_j x_i(n)|^2 = 2 \sum_i \|S_j x_i \|^2.
\]
It results that 
a sparse representation along the index $i$ implies
that $H_j S_j x_i (n)$ is also sparse along $n$. The same type of result is
thus obtained by replacing the 
mixed $\bf l^1$ and $\bf l^2$ norm (\ref{l1energy0}) 
by a simpler $\bf l^1$ sparsity norm along both the $i$ and $n$ variables
\begin{equation}
\label{l1energy}
\sum_n \sum_i  | H_j S_j x_i(n) | ~.
\end{equation}
This sparsity norm is often used by sparse autoencoders for unsupervised
learning of deep networks \cite{bengio2013representation}.
Numerical results in Section \ref{sec:exp} verify 
that both norms have very close classification performances.

For Haar operators $H_j$, 
the $\bf l^1$ norm leads to a simpler interpretation of the
result. Indeed a Haar filtering is defined by a pairing $\pi_j$ of $d$ integers $\{1,...,d\}$. 
Optimizing $H_j$  amounts to optimize $\pi_j$, and hence minimize
\[
\sum_n \sum_i |H_j S_j x_i(n)| = 
\sum_n \Big(S_j x_i (\pi_j(2n))+S_j x_i (\pi_j(2n+1)) + |S_j x_i (\pi_j(2n))-S_j x_i (\pi_j(2n+1))| 
\Big).
\]
But $\sum_n (S_j x (\pi_j(2n))+S_j x (\pi_j(2n+1))) = \sum_n S_j x(n)$ does not  depend upon the pairing $\pi_j$.  
Minimizing the $\bf l^1$ norm (\ref{l1energy}) is thus equivalent to minimizing
\begin{equation}
\label{l1energy2}
\sum_n \sum_i 
|S_j x_i (\pi_j(2n))-S_j x_i (\pi_j(2n+1))| .
\end{equation}
It minimizes the average variation within pairs, and thus tries to
regroup pairs having close values.

Finding a linear operator $H_j$ which minimizes (\ref{l1energy0}) or (\ref{l1energy}) is a ``dictionary learning''
problem which is in general an NP hard problem. For a Haar dictionary, we 
show that it is equivalent to a pair matching problem and can thus be solved with
$O(d^3)$ operations. For both optimization norms, it amounts to finding
a pairing $\pi_j$ which minimizes an additive cost
\begin{equation}
C(\pi_j) = \sum_n C(\pi_j(2n),\pi_j(2n+1)),
\end{equation}
where $C(\pi_j(2n), \pi_j(2n+1)) = \sum_i |H_j S_j x_i (n)|$ for
(\ref{l1energy}) and 
$C(\pi_j(2n), \pi_j(2n+1)) = \Big(\sum_i |H_j S_j x_i (n)|\Big)^2$ 
for (\ref{l1energy0}). 
This linear pairing cost is minimized exactly by the Blossom algorithm with $O(d^3)$ operations. Greedy method obtains a $1/2$-approximation in $O (d^2)$ time \cite{match-fast-square}. Randomized approximation similar to \cite{jones2011randomized} could also be adapted to achieve a complexity of $O ( d \log d)$ for very large size problems.

Theorem \ref{thm:multi-layer-recovery2} proves that several 
Haar scattering transforms are necessary to obtain a complete signal representation. We learn $T$ Haar scattering transforms by dividing the training set
$\{x_i \}_i$ in $T$ non-overlapping subsets. A different Haar scattering transform 
is optimized for each training subset. 
Next section describes a supervised classifier applied to the resulting
bag of $T$ Haar scattering transforms. 

\subsection{Supervised feature selection and classification}
\label{supsec}

Strong invariants are computed by the supervised classifier which essentially
computes adapted linear combinations of Haar scattering coefficients.
Bagging $T$ orthogonal Haar scattering representations defines a set
of $T \,d$ scattering coefficients. A supervised dimension reduction is first
performed by selecting a subset of scattering coefficients.
It is implemented with an orthogonal least square forward selection algorithm \cite{chen1991orthogonal}. The final supervised classification is implemented with a Gaussian kernel SVM classifier applied to this reduced set of coefficients.

We select $K$ scattering coefficient to discriminate each class $c$ from all other classes, and decorrelate these features before applying the SVM classifier.
Discriminating a class $c$ from all other classes amounts to approximating 
the indicator function
\[
f_c(x) = 
\left\{
\begin{array}{ll}
1 & \mbox{if $x$ belongs to class $c$}\\
0 & \mbox{otherwise}
\end{array}
\right.~.
\]

Let us denote by $\Phi x = \{ \phi_p x \}_{p \leq T d}$ the dictionary of $T\,d$ scattering coefficients to which is added the constant $\phi_0 x = 1$.  An orthogonal least square linearly approximates $f_c(x)$ with a sparse
subset of $K$ scattering coefficients $\{ \phi_{p_k} \}_{k \leq K}$ which are 
greedily selected
one at a time. To avoid correlations between selected features, it includes
a Gram-Schmidt orthogonalization which 
decorrelates the scattering dictionary 
relatively to previously selected features. We denote by
$\Phi^k x = \{ \tilde \phi^k_{p} x \}_p$ the scattering dictionary, which was
orthogonalized and hence decorrelated relatively to the
first $k$ selected scattering features. 
For $k = 0$, we have $\Phi^0 x = \Phi x$. At the $k+1$ iteration, we select
$\phi^k_{p_{k}} x \in \Phi^k x$ which yields the minimum linear mean-square error over
training samples:
\begin{equation}
\label{error}
\sum_i \Big( f_c (x_i) - \sum_{\ell = 0}^{k} \alpha_\ell \, \phi^\ell_{p_{\ell}} x_i \Big)^2
\end{equation}
Because of the orthonormalization step, the linear regression coefficients are
\[
\alpha_\ell = \sum_i f_c(x_i) \, \phi^\ell_{p_{\ell}} x_i 
\]
and
\[
\sum_i \Big( f_c (x_i) - \sum_{\ell = 0}^{k} \alpha_\ell \, \phi^\ell_{p_{\ell}} x_i \Big)^2 
= \sum_i |f_c (x_i)|^2 - \sum_{\ell=0}^{k} \alpha_\ell^2~.
\]
The error (\ref{error}) is thus minimized by choosing $\phi^k_{p_{k+1}} x$ 
having a maximum correlation:
\[
\alpha_{k} = \sum_i f_c (x_i)\, \phi^k_{p_{k}} x_i = \arg \max_{p}\Big( \sum_i f_c (x_i)\, \phi^k_{p} x_i 
\Big)~.
\]
The scattering dictionary is then updated by orthogonalizing each of its element
relatively to the selected scattering feature $\phi^k_{p_{k} x}$:
\[
\phi^{k+1}_{p} x = \phi^{k}_{p} x -  \Big(\sum_i \phi^k_p x_i\, \phi^k_{p_{k}} x_i \Big)\, 
\phi^k_{p_{k}} x~.
\]

This orthogonal least square regression greedily selects the $K$ decorrelated scattering  features $\{ \phi^k_{p_k} x \}_{0\leq k < K}$ for each class $c$. 
For a total of $C$ classes, the union of all these features defines  a dictionary of size $M = K\, C$. They are linear combinations of the original Haar scattering coefficients $\{ \phi_p x\}_p$. In the context of a deep neural network,
this dimension reduction can be interpreted as a last fully connected network layer, which takes in input $T\,d$ scattering coefficients
and outputs a vector of size $M$.  The parameter $M$ optimizes the  bias versus variance trade-off. It may be set a priori or 
adjusted by cross validation in order to yield a minimum classification error at the output of the Gaussian kernel SVM classifier. 

A Gaussian kernel SVM classifier is applied to the $M$-dimensional orthogonalized scattering feature vectors. The Euclidean norm of this vector is normalized to $1$.
In the applications of Section 
\ref{sec:exp}, $M$ is set to $10^3$ and hence remains large.
Since the feature vectors lies on a high-dimensional unit sphere,
the standard deviation $\sigma$ of the 
Gaussian kernel SVM must be of the order of $1$. Indeed, a Gaussian kernel SVM
performs its classification by 
fitting separating hyperplane over different
balls of radius of radius $\sigma$. If $\sigma \ll 1$ then the number 
balls covering
the unit sphere grows like $\sigma^{-M}$. Since $M$ is large,
$\sigma$ must remain in the order of $1$ to insure 
that there are enough training samples to fit a hyperplane in each ball.

\section{Orthogonal Haar Scattering on Graphs}\label{sec:3}

Signals such as images are sampled on uniform grids.
Many classification problems are  translation invariant, which motivates
the calculation of translation invariant representations.
A translation invariant representation
can be computed by averaging signal samples, but it removes too much 
information. Wavelet scattering operators \cite{mallat2012group} are calculated
by cascading wavelet transforms and absolute values. Each wavelet transform
computes multiscale signal variations on the grid. It yields
a large vector of coefficients, whose spatial averaging defines a rich set of 
translation invariant coefficients.

Data vectors may be defined on non-uniform graphs
\cite{shuman2013emerging}, for example in social, financial or transportation
networks. A graph displacement moves data samples on the grid but is not 
equivalent to a uniform grid translation. Orthogonal Haar scattering
transforms on graphs are computed from local multiscale signal variations
on the graph. The calculation of displacement invariant features is left to the
final supervised classifier, which adapts the averaging to 
the classification problem.
Section \ref{stransdsec} introduces this Har scattering on a graph as a particular
case of orthogonal Haar scattering. 
Section \ref{Haarwave} proves that an orthogonal Haar scattering
on graphs can be written as a product of wavelet transforms, as usual
wavelet scattering operators. When the graph connectivity is unknown,
unsupervised learning can
calculate a Haar scattering on the unknown graph and estimate the graph
connectivity. 
The consistency of such estimations is studied in  
Section \ref{subsec:learn-connect}.

\subsection{Structured orthogonal Haar scattering}
\label{stransdsec}

\begin{figure}
\begin{center}
\includegraphics[width=0.65\linewidth]{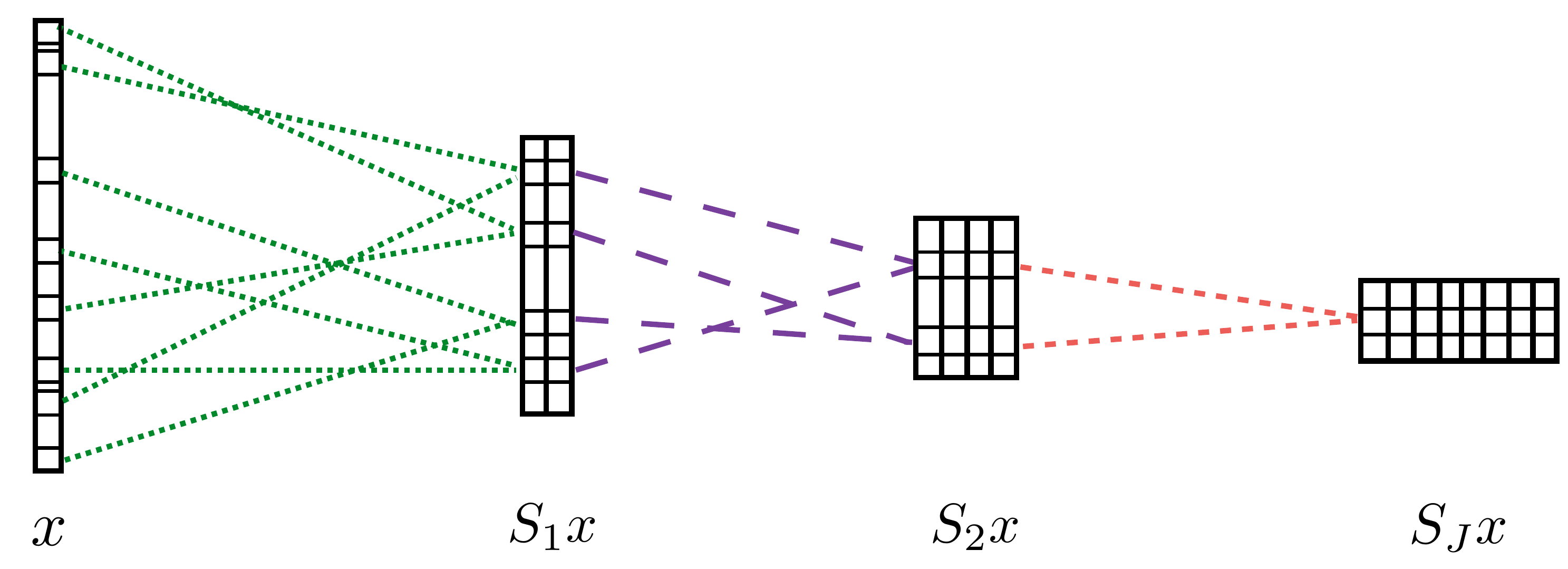}
\caption{
\small A structured Haar scattering on a graph 
computes each layer $S_{j+1}x$ by pairing the rows of the previous
layer $S_j x$. For each pair of rows, it stores
their sum and the absolute values of their difference, in a twice bigger row.
}
\label{fig:1-2}
\end{center}
\end{figure}

\begin{figure}[h!]
    \centering
    \begin{subfigure}[b]{0.55\textwidth}
        \centering
        \includegraphics[width=1\linewidth]{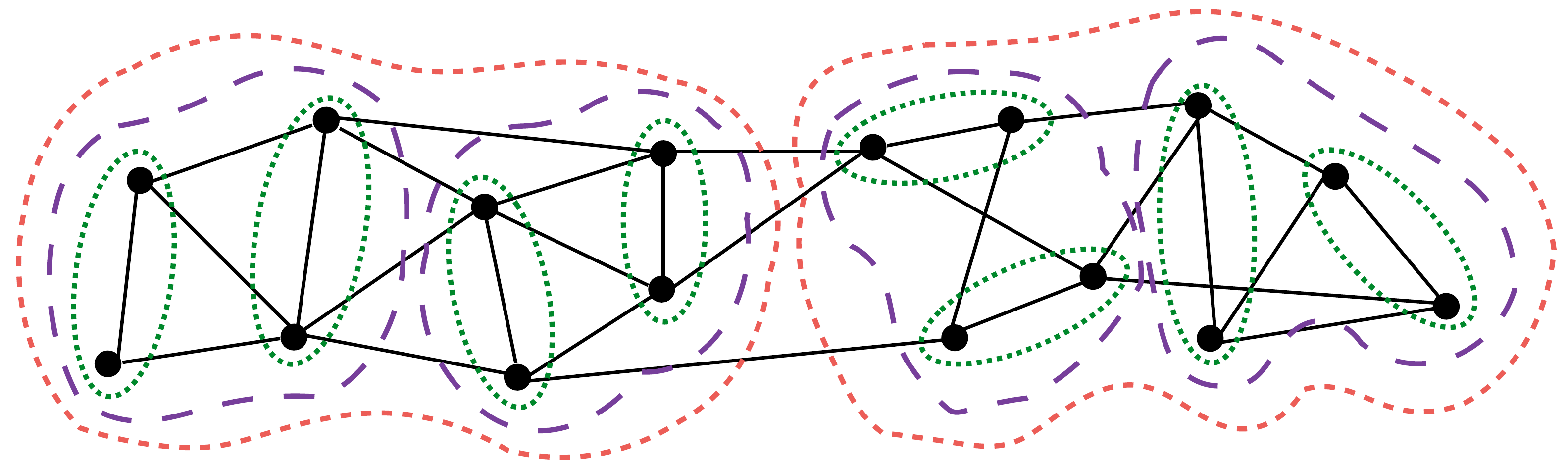} 
         \caption{ 
	\small
	}
    \end{subfigure}%
   
    \begin{subfigure}[b]{0.55\textwidth}
        \centering
         \includegraphics[width=1\linewidth]{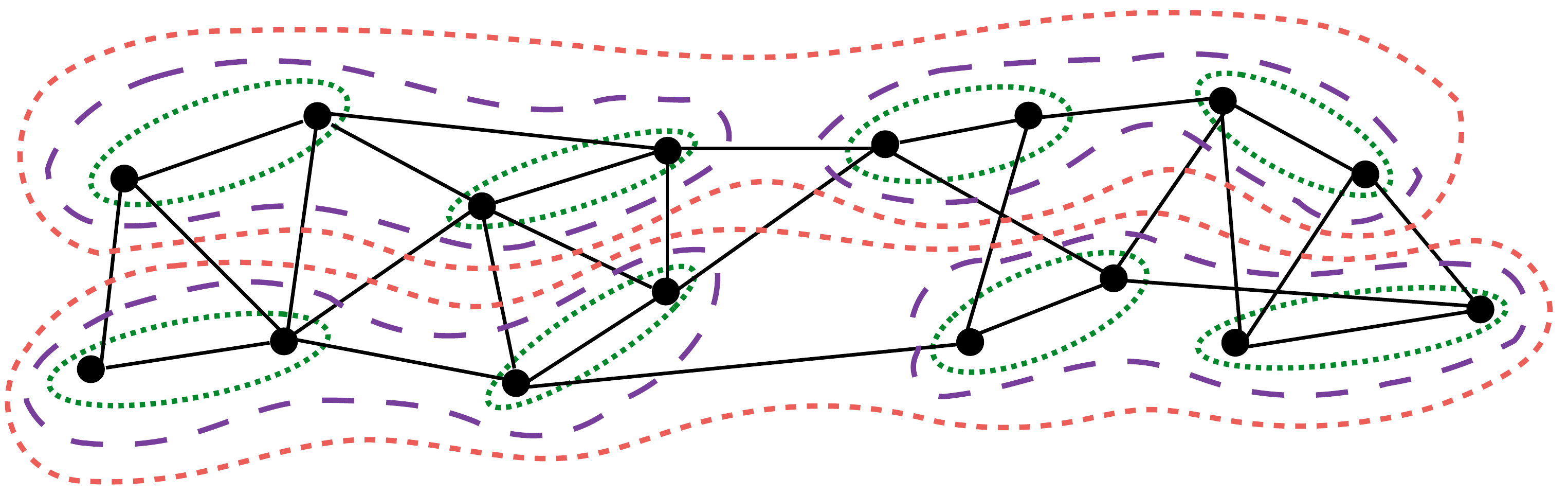} 
         \caption{ 
	\small
	}
    \end{subfigure}

    \begin{subfigure}[c]{0.30\textwidth}
        \centering
         \includegraphics[width=1\linewidth]{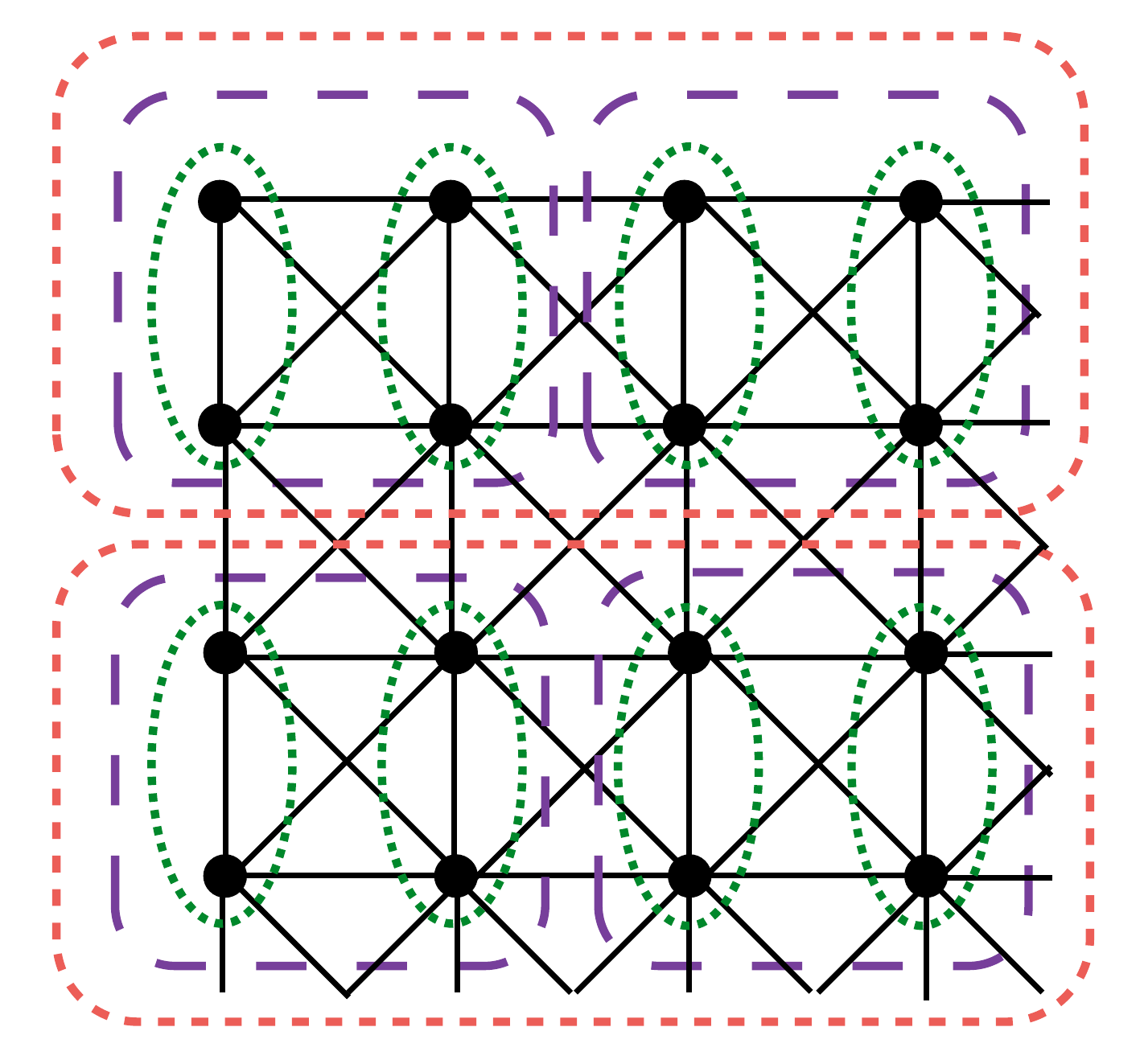} 
         \caption{ 
	\small
	}
    \end{subfigure}
 \caption{\label{fig:2}
\small (a,b): Two different examples of
hierarchical partitions of a graph into connected sets $V_{j,n}$ of size $2^j$, 
for $j=1$ (green), $j=2$ (purple) and $j=3$ (red).
(c) Hierarchical partitions on a square image grid. 
 }
\end{figure}

The free orthogonal Haar scattering transform of Section \ref{sec:2} freely associates any two elements of an internal network layer $S_j x$. A Haar scattering on a graph is constructed by pairing elements according to their position in the
graph, which requires to structure the pairing and the network layers.
We denote by $V$ the set of $d$ vertices of this graph, and assume that $d$ is a power of 2.
The vector $S_j x$ of size $d$ is structured as a two-dimensional  array  $S_j x (n,q)$ of size $2^{-j} d \times 2^j$. 
For each $j \geq 0$, we shall see that $n\in \{1,...,2^{-j}d\}$ is 
a ``spatial'' index of a set of $V_{j,n}$ of $2^j$ graph vertices. 
The $2^j$ parameters $q$ are indexing different
permutation invariant coefficients computed from the values of $x$
in  $V_{j,n}$.
The input network layer is $S_0 x (n,0) = x(n)$. 
We compute $S_{j+1} x$ by pairing the $2^{-j} d$ rows of
$S_{j} x$. The row pairing
\begin{equation}
\label{pairing}
\pi_j = \Big\{ (\pi_j(2n),\pi_j(2n+1)) \Big\}_{0\leq n < 2^{-j} d}. 
\end{equation}
is pairing each $(\pi_j(2n),q)$ with $(\pi_j(2n+1),q)$ for $0\leq q < 2^j$. 
It imposes a row structure on the free pairing of Section \ref{freesec}.
Applying the absolute Haar filter (\ref{permasn}) to each pair gives
\begin{equation}
\label{eqn12}
S_{j+1} x (n,2q) = S_j x(\pi_j(2n),q) + S_j x(\pi_j(2n+1),q)
\end{equation}
and
\begin{equation}
\label{eqn22}
S_{j+1} x(n,2q+1) = |S_j x(\pi_j(2n),q) - S_j x(\pi_j(2n+1),q)|~.
\end{equation}
Applying these equation for $j \leq J$ defines a
structured Haar network illustrated in Figure \ref{fig:1-2}.
If we remove the absolute value from (\ref{eqn22}) then these equations
iterate linear Haar filters and define an orthogonal Walsh transform
\cite{Coifman}. The absolute value completely modifies the properties of
this transform but Section \ref{Haarwave} proves that it can be still be written
as a product of orthogonal Haar wavelet transforms, 
alternating with absolute value non-linearities.

The following proposition proves that this structured Haar scattering is a transformation
on a hierarchical grouping of the graph vertices, derived 
from the row pairing (\ref{pairing}).
Let $V_{0,n} = \{n\}$ for $n \in V$. For any 
$j \geq 0$ and $n \in \{1,...,2^{-j-1}d \}$, we define
\begin{equation}
\label{regroup}
V_{j+1,n} = V_{j,\pi_j(2n)} \cup V_{j,\pi_j(2n+1)}.
\end{equation}
We verify by induction on $j$ that it defines a partition
$V = \cup_n V_{j,n}$, where each $V_{j,n}$ is a set of $2^j$ vertices.

\begin{proposition}\label{thm:VJN}
The coefficients $\{ S_J x(n,q) \}_{0 \le q < 2^j}$ are computed by applying a Hadamard matrix
to the restriction of $x$ to $V_{J,n}$. This Hadamard matrix depends on $x$, $J$ and $n$. 
\end{proposition}

\begin{proof}
Theorem \ref{prop:norm-preserve} proves that  $\{ S_{J}x(n,q) \}_{0 \le q <2^J} $ is computed by applying am orthogonal 
transform to $x$. To prove that it is a Hadamard matrix, 
it is sufficient to show that its  entries are $\pm 1$. 
We verify by induction on $j\leq J$ that $S_j x(n,q)$ only depends
on restriction of $x$ to $V_{j,n}$,  by applying
(\ref{eqn22}) and (\ref{eqn12}) together with (\ref{regroup}).
We also verify 
that each $x(v)$ for $v \in V_{j,n}$ appears exactly once in the calculation, 
with an addition or a subtraction.
Because of the absolute value, the addition or subtraction which are $1$ and
$-1$ in the Hadamard matrix, which therefore depends upon $x$, $J$ and $n$
\end{proof}

An orthogonal Haar scattering on a graph can thus be interpreted as an
adaptive Hadamard transform over groups of vertices, which outputs positive 
coefficients. Walsh matrices are particular cases of Hadamard matrices.
The induction (\ref{regroup}) defines sets $V_{j,n}$ with connected
nodes in the graph if for 
all $j$ and $n$, each pair $(\pi_j(2n),\pi_j(2n+1))$ regroups two 
sets $V_{j,\pi_j(2n)}$ and $V_{j,\pi_j(2n+1)}$ which are connected. It means that 
at least one element of $V_{j,\pi_j(2n)}$ is connected to one element of $V_{j,\pi_j(2n+1)}$. 
There are many possible connected dyadic partitions of any given graph. Figure \ref{fig:2}(a,b) shows two different examples of connected graph partitions.

For images sampled on a square grid, a pixel is connected with $8$ neighbors.
A structured Haar scattering can be computed by pairing 
neighbor image pixels, alternatively along rows and columns as the depth $j$
increases. When $j$ is even, each $V_{j,n}$ 
is then a square group of $2^j$ pixels, as illustrated in 
Figure \ref{fig:2}(c).
Shifting such a partition defines a new partition.
Neighbor pixels can also be grouped in the diagonal direction which amounts
to rotate the sets $V_{j,n}$ by $\pi/4$ to define a new dyadic partition. 
Each of these partitions define a different structured Haar scattering.
Section \ref{sec:exp} applies these
structured Haar image scattering to image classification.

\subsection{Scattering order}
\label{ordersec}

Scattering coefficients have very different properties depending upon the number of absolute values which are used to compute them.  A scattering coefficient of order $m$ is a coefficient computed by cascading $m$ absolute values.
Their amplitude have a fast decay as the order $m$ increases, and
their locations are specified by the following proposition.  

\begin{proposition}
\label{musdfords}
If $q = 0$ then $S_j x(n,q)$ is a coefficient of order $0$.
Otherwise, $S_j x(n,q)$ is a coefficient of order $m \leq j$ if
there exists $0 \leq j_1 < ...<j_m < j$ such that
\begin{equation}
\label{coeffssd}
q = \sum_{k=1}^m 2^{j-j_k}~.
\end{equation}
There are $ {j  \choose m}  2^{-j} d$ coefficients of order $m$ in $S_j x$.
\end{proposition}

\begin{proof}
This proposition is proved by induction on $j$.
For $j=0$ 
all coefficients are of order $0$ since $S_0 x(n,0) = x(n)$. 
If $S_j x(n,q)$ is of order $m$ then 
(\ref{eqn12}) and (\ref{eqn22}) 
imply that $S_{j+1} x(n,2q)$ is of order $m$ and $S_{j+1} x(n,2q+1)$ is of order $m+1$. It results that (\ref{coeffssd}) is valid for $j+1$ if is valid for $j$.

The number of coefficients $S_j x(n,q)$ of order $m$ corresponds to the number of
choices for $q$ and hence for $0 \leq j_1 < ...<j_m < j$,
which is $ {j  \choose m}$. This must be multiplied by the number of indices
$n$ which is $2^{-j}d$. 
\end{proof}

The amplitude of scattering coefficients typically decreases exponentially when the scattering order $m$ increases, because of the contraction produced by the absolute value. High order scattering coefficients can thus be neglected. 
This is illustrated by considering a vector $x$ of independent Gaussian random variables of variance $1$. The value of $S_j x(n,q)$ only depends upon the values of $x$ in $V_{j,n}$. 
Since $V_{j,n}$ does not intersect with
$V_{j,n'}$ if $n \neq n'$, we derive that $S_j (n,q)$ and $S_j(n',q)$ are independent.
They have same mean and same variance because $x$ is identically 
distributed. Scattering coefficients
are iteratively 
computed by adding pairs of such coefficients, or by computing the absolute value of
their difference. Adding
two independent random variables multiplies their variance by $2$. Subtracting two
independent random variables of same mean and variance yields a new random
variable whose mean is zero and whose variance 
is multiplied by $2$. Taking the absolute value reduces the variance by a factor
which depends upon its probability distribution. If this distribution is Gaussian then
this factor is $1 - 2/\pi$. 
If we suppose that this distribution remains approximately Gaussian, then applying $m$ absolute values reduces the variance by approximately $(1-2/\pi)^m$. 
Since there are $J \choose{m}$ coefficients of order $m$, their total normalized variance  $\sigma^2_{m,J}$ is approximated by ${J \choose{m}} (1-2/\pi)^m$. 
Table \ref{table:GaussianVar} shows that ${J \choose {m}} (1 - 2 / \pi)^m$ is 
indeed of the same order of magnitude as the value
$\sigma_{m,J}^2$ computed numerically.  This variance becomes much smaller for $m > 4$. This observation remains valid for large classes of signals $x$. Scattering coefficients of order $m > 4$ usually have a negligible energy and are thus removed in classification applications.

\begin{table}[h!]
\begin{center}
\begin{tabular}{c|c c c c c}
\hline
      $m$           &       1        &      2        &       3                           &        4                                 &       5  \\
\hline
 $\sigma^2_{m,J} $  
			&   1.8    &   1.4   &  $5.8\times 10^{-1}$ &    $1.2 \times 10^{-1} $ &  $1.2 \times 10^{-2} $ \\
$(1-\frac{2}{\pi})^m\cdot {J \choose m} $ 
                       &   1.8    &    1.3  &   $4.8 \times 10^{-1}$ &  $8.7 \times 10^{-2}$ &   $ 6.3 \times 10^{-3} $ \\
\hline
\end{tabular}
\end{center}\vspace{-0.25cm}
\caption{ \label{table:GaussianVar}
\small  $\sigma_{m,J}^2$ is the normalized variance of all order $m$ coefficients in $S_J x$, computed for a Gaussian white noise $x$ with $J = 5$. It is decays
approximately like $ (1-\frac{2}{\pi})^m\cdot {J \choose m} $.
}
\vspace{-0.25cm}
\end{table}

\subsection{Scattering with orthogonal Haar wavelet bases}
\label{Haarwave}

We now prove that scattering coefficients of order $m$ are obtained by cascading $m$ orthogonal Haar wavelet transforms defined on the graph. Haar wavelets
can easily be constructed on graphs 
\cite{  gavish2010multiscale, rustamov2013wavelets}.
Section \ref{stransdsec} shows that a Haar scattering on a graph
is constructed over dyadic partitions $\{ V_{j,n} \}_n$ of $V$, which are obtained by progressively aggregating vertices by pairing
$V_{j+1,n} = V_{j,\pi_j(2n)} \cup V_{j,\pi_j(2n+1)}$.
We denote by $1_{V_{j,n}}(v)$ the indicator function of $V_{j,n}$ in $V$. A Haar wavelet computes the difference between the sum of signal values over two aggregated sets: 
\begin{equation}\label{eq:psi-jn}
\psi_{j+1,n} = 1_{V_{j,\pi_j(2n)}} - 1_{V_{j,\pi_j(2n+1)}}~. 
\end{equation}
Inner products between signals defined on $V$ are written
\[
\lb x , x' \rb = \sum_{v \in V} x(v)\, x'(v). 
\]
For any $2^J < d$, 
\begin{equation}\label{eq:Haar-wavelet-j}
\{1_{V_{J,n}} \}_{0 \leq n < 2^{-J } d} \cup \{\psi_{j,n}\}_{0 \leq n < 2^{-j} d , 0 \leq j < J}
\end{equation}
is a family of $d$ orthogonal Haar wavelets which define an
orthogonal basis of $\R^d$.
The following theorem proves that order $m+1$ coefficients are obtained
by computing the orthogonal Haar wavelet transform of coefficients of order $m$. The proof is in Appendix \ref{app:1}.

\begin{theorem}\label{prop:order-m-coeff}
Let $q = \sum_{k=1}^{m} 2^{j-j_k}$ with $j_1< ... < j_m\leq j$. 
If $j_{m+1} > j_m$ then for each $n \leq 2^{-j-1} d$
\begin{equation}
\label{propsdfnsd}
S_j x( n,  q + 2^{j-j_{m+1}}  )
 = \sum_{ p \atop{V_{j_{m+1}, p}\subset V_{j, n}} }
	| \lb \overline{S}_{j_{m}}  x (\cdot,  2^{j_m-j} q ) , \psi_{j_{m + 1},  p}   \rb |.
\end{equation}
with
\[
\overline{S}_{j_m} x ( . , q') = \sum_{n=0}^{2^{-j_m}d-1} S_{j_m} x(n, q')\, 1_{V_{j_m,n}}.
\]
\end{theorem}

If $q = \sum_{k=1}^{m} 2^{j-j_k}$ and $j_{m+1} > j_m$ then
$S_{j_m} x(n,2^{j_m-j}q)$ are coefficients of order $m$ whereas
$S_j x(n,q+ 2^{j-j_{m+1}}  )$ is a coefficient
of order $m+1$. Equation (\ref{propsdfnsd}) proves that a coefficient
of order $m+1$ is obtained by calculating the wavelet transform of
scattering coefficients of order $m$, and summing their absolute values.
A coefficient of order $m+1$ 
thus measures the averaged variations of the $m$-th order scattering coefficients 
on neighborhoods of size $2^{j_{m+1}}$ in the graph. For example, if $x$ is constant in a $V_{j,n}$ then $S_\ell x(n,q) = 0$ if $\ell \le j$ and $q \neq 0$.

\subsection{Learning graph connectivity by variation minimization}
\label{subsec:learn-connect}

In many problems the graph connectivity is unknown. 
Learning a connected dyadic partitions is easier than learning the full connectivity of a graph, which is typically an NP complete problem.  Section \ref{subsec:unsup-learn} introduces a  polynomial complexity  algorithm,
which learns pairings in orthogonal Haar scattering networks. 
For a Haar scattering on a graph, we show that this algorithms amounts to computing dyadic partitions where scattering coefficients have a minimum total variation. 
The consistency of this pairing algorithm is studied over particular
Gaussian stationary processes, and we show that there is no curse of 
dimensionality.

Section \ref{subsec:unsup-learn} introduces two criteria 
to optimize the pairing $\pi_j$ of a free orthogonal Haar 
scattering, from a training set $\{x_i \}_{i\leq N}$. 
We concentrate on the $\bf l^1$ 
norm minimization, which has a simpler expression. 
For a Haar scattering on a graph, the
$\bf l^1$ minimization (\ref{l1energy2}) computes 
a row pairing $\pi_j$ which minimizes
\begin{equation}
\label{msdifnsdf}
\sum_{i=1}^N \, \sum_{n=0}^{d 2^{-(j+1)}}\, \sum_{q=0}^{2^j-1} 
|S_j x_i (\pi_j(2n),q) - S_j x_i (\pi_j(2n+1),q)|.
\end{equation}
This optimal pairing regroups vertex sets $V_{j,\pi_j(2n)}$ and $V_{j,\pi_j(2n+1)}$ whose scattering coefficients have a minimum total variation.

Suppose that the $N$ training samples $x_i$ are independent realizations of a random vector $x$. To guarantee that this pairing finds connected sets we must make sure that the total variation minimization favors regrouping neighborhood points, which means that $x$ has some form of regularity on the graph. 
We also need $N$ to be sufficiently large
so that this minimization finds connected sets with high probability, despite statistical fluctuations. Avoiding the curse of dimensionality means that $N$ should not grow exponentially with the signal dimension $d$. 

To attack this problem mathematically, we consider  a very particular case, where signals are defined on a ring graph, and are thus $d$ periodic. Two indices $n$ and $n'$ are connected if $|n-n'| = 1\mod d$. We study the optimization of the first network layer for $j = 0$, where $S_0 x(n,q) = x(n)$. The minimization of (\ref{msdifnsdf}) amounts to  compute a pairing $\pi$ which minimizes
\begin{equation}
\label{msdifnsdf2}
\sum_{i=1}^N \left(\sum_{n=0}^{d/2-1} |x_i (\pi(2n)) - x_i (\pi(2n+1))| \right).
\end{equation}
This pairing is connected if and only if for all $n$, $|\pi(2n)- \pi(2n+1)| = 1 \mod d$.

\begin{figure}[h!]
\centering
\includegraphics[width= 0.45 \textwidth ]{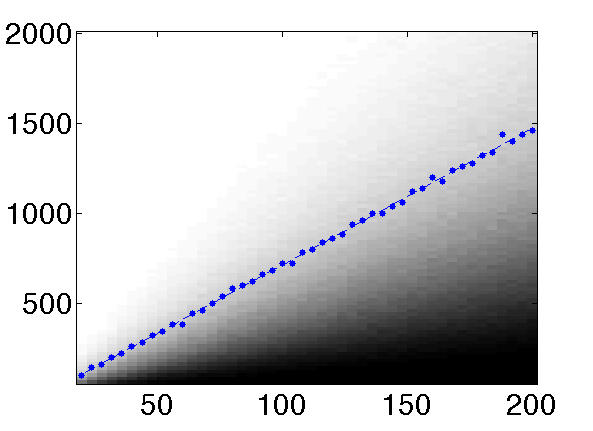}
\caption{ \small Each image pixel gives the probability that
the  total variation minimization (\ref{msdifnsdf2}) finds pairs which
are all connected, when $x$ is a Gaussian stationary vector. It is computed as
a function of the dimension $d$ of the vector (horizontal axis) and of the
number $N$ of training samples (vertical axis). Black and white 
points are probabilities respectively equal to $0$ and $1$. 
The blue dotted line corresponds to a probability $0.8$.
}
\label{Gaussian}
\end{figure}

The regularity and statistical fluctuations of $x(n)$ are controlled by supposing that $x$ is a circular stationary Gaussian process. The stationarity implies that its covariance matrix ${\rm Cov} (x(n),x(m)) = \Sigma(n,m)$ depends on the
distance between points $\Sigma(n,m) = \rho((n-m) \mod d)$. The average regularity depends upon the decay of the correlation $\rho(u)$.  We denote by $\|\Sigma \|_{op}$ the sup operator norm of $\Sigma$.  The following theorem proves that the training size $N$ must grow like  $d\, \log d$ in order to compute an optimal pairing with a high probability. The constant is inversely proportional to a normalized ``correlation gap,'' which depends upon
the difference between the correlation of neighborhood points and more far away points.
It is defined by
\begin{equation}
\quad
\Delta = \left(  \sqrt{ 1- \frac{
{ \max_{ n \geq 2} \rho(n)}}{\rho(0)}}  -  \sqrt{ 1 - \frac{\rho(1)}{\rho(0)}}    
\right)^2.
\label{eq:decay-rho}
\end{equation}

\begin{theorem}\label{thm:n>dlogd}
Given a circular stationary Gaussian process with $\Delta > 0$, the pairing which minimizes the empirical total variation (\ref{msdifnsdf2}) has probability larger than $1 - \epsilon$ to be connected if
\begin{equation}
\label{theoansdf}
N >\frac{\pi^3 \|\Sigma \|_{op}}{2 \Delta}\, d\Big(3\log d - \log{\varepsilon} \Big).
\end{equation}
\end{theorem}

The proof is based on the Gaussian concentration inequality for Lipschitz function \cite{maurey1991, pisier1985} and is left to Appendix  \ref{app2} . Figure \ref{Gaussian} displays numerical results obtained with a Gaussian stationary process of dimension $d$ where  $ \rho(1)/\rho(0)  =  0.44$ and $\max_{ n \geq 2} \rho(n)/
\rho(0) = 0.06 $.  The gray level image gives the probability that a pairing is connected when computing this pairing by minimizing the total variation (\ref{msdifnsdf2}), as a function of the dimension $d$ and of the number $N$ of training samples. The black and white points correspond to probabilities $0$ and $1$ respectively. In this example, we see that the optimization gives a connected pairing with probability $1-\epsilon$ for $N$ increasing almost linearly with $d$, which is illustrated by the nearly straight line of dotted points corresponding to $\epsilon = 0.2$. The Theorem gives an upper bound which grows like $d \, \log d$, though the constant involved is not tight.

For layer $j> 1$,  $S_j x(n,q)$ is no longer a Gaussian random vector
due to the absolute value non-linearity. However, the  result can be extended using a Talagrand-type concentration argument instead of the Gaussian concentration. Numerical experiments presented in Section \ref{sec:exp} show that this approach does recover the connectivity of high dimensional images with a probability close to $100 \%$ for $j \leq 3$, and that the probability decreases as $j$ increases. 
This seems to be due to the fact that the absolute value contractions reduce
the correlation gap $\Delta$ 
between connected coefficients and more far away coefficients when $j$
increases.

\section{Numerical classification experiments}\label{sec:exp}

Haar scattering representations are tested on classification problems, over
images sampled on a regular grid or an irregular graph. 
We consider the cases where the grid or the graph geometry is known a priori, 
or discovered by unsupervised learning. 
The efficiency of
free and structured Haar scattering architectures 
are compared with state of the art classification results obtained
by deep neural network learning, when the graph geometry is known or unknown.
Although computations are reduced to additions and subtractions, we show a Haar scattering can get state art results when the graph
geometry is unknown even over complex image data bases. For images 
over a known uniform sampling grid, we show 
that the simplifications of a Haar scattering produces an error about
$20\%$ larger than state of the art unsupervised learning algorithms.

A Haar scattering classification involves few parameters which are reviewed. The scattering scale $2^J \leq d$ is the permutation
invariance scale. Scattering coefficients are computed up to the a maximum order 
$m$, which is set to $4$ in all experiments. Indeed, higher order scattering coefficient have a negligible relative energy, which is below $1\%$, as explained 
in Section \ref{ordersec}.
The unsupervised learning algorithm computes $T$ 
different Haar scattering transforms by subdividing the training set in
$T$ subsets. Increasing $T$ decreases
the classification error but it increases computations.
The error decay becomes negligible for $T \geq 40$. 
The supervised dimension reduction selects a final 
set of $M$ orthogonalized scattering coefficients. We set  $M = 1000$ in all numerical experiments.

\subsection{Classification of image digits in MNIST}
\label{original}

\begin{figure}[t]
\begin{center}
\begin{subfigure}[b]{0.95\linewidth}
\includegraphics[width=0.09\linewidth]
{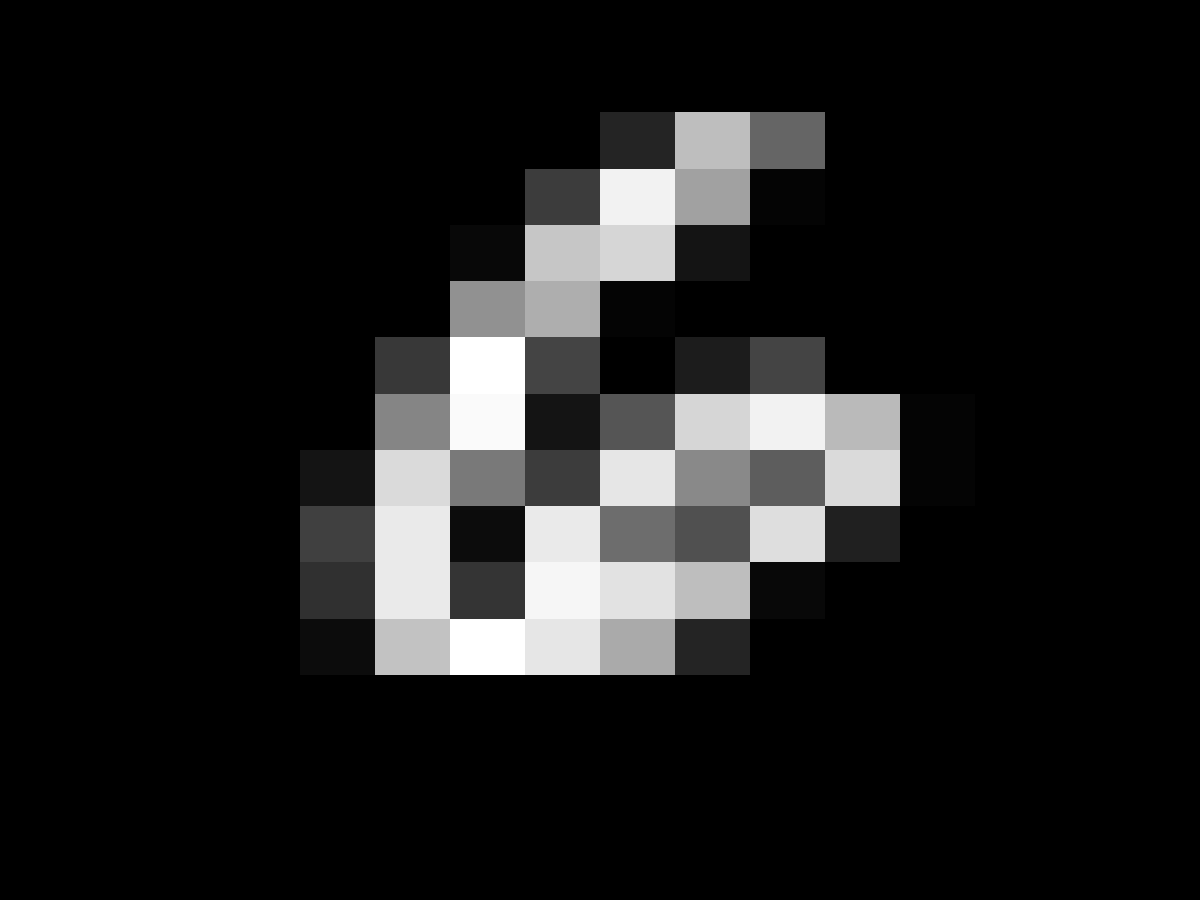}
\includegraphics[width=0.09\linewidth]
{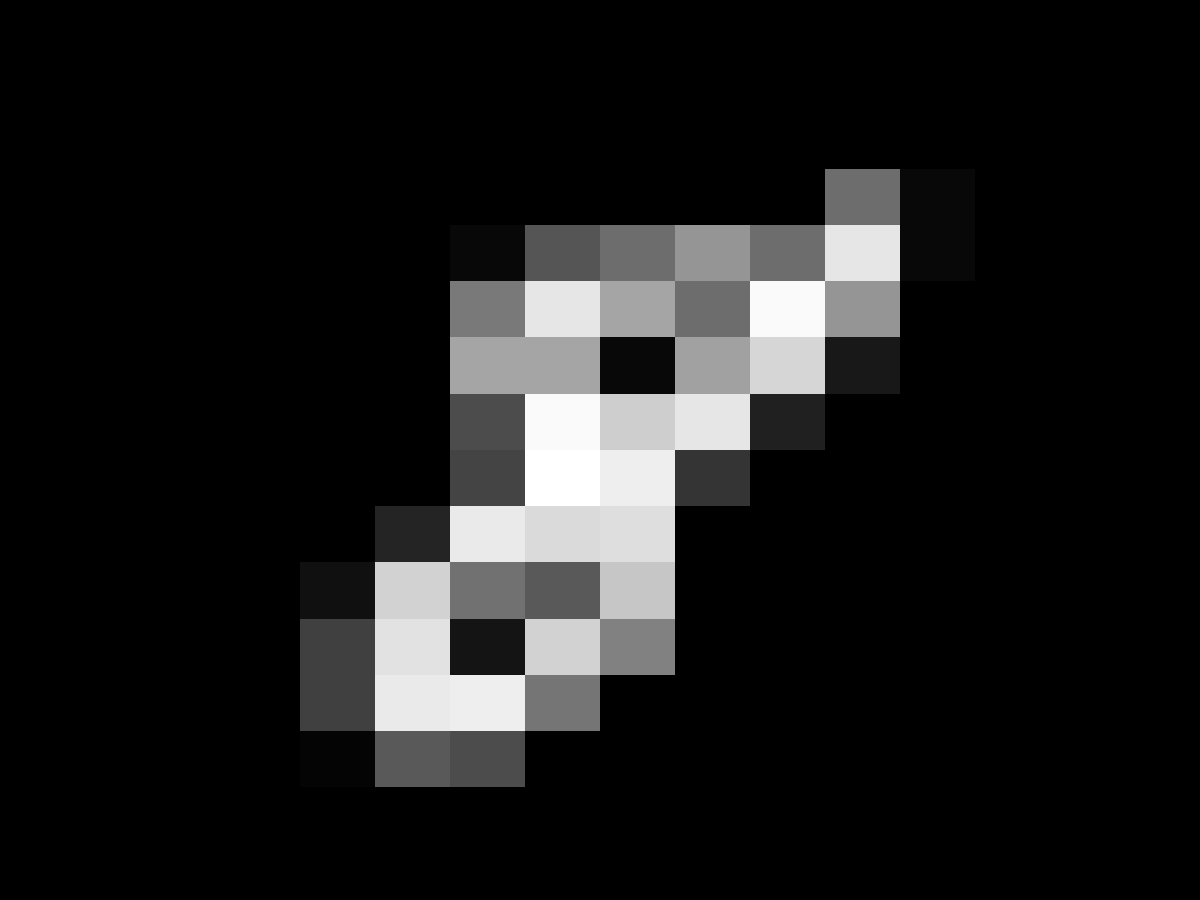}
\includegraphics[width=0.09\linewidth]
{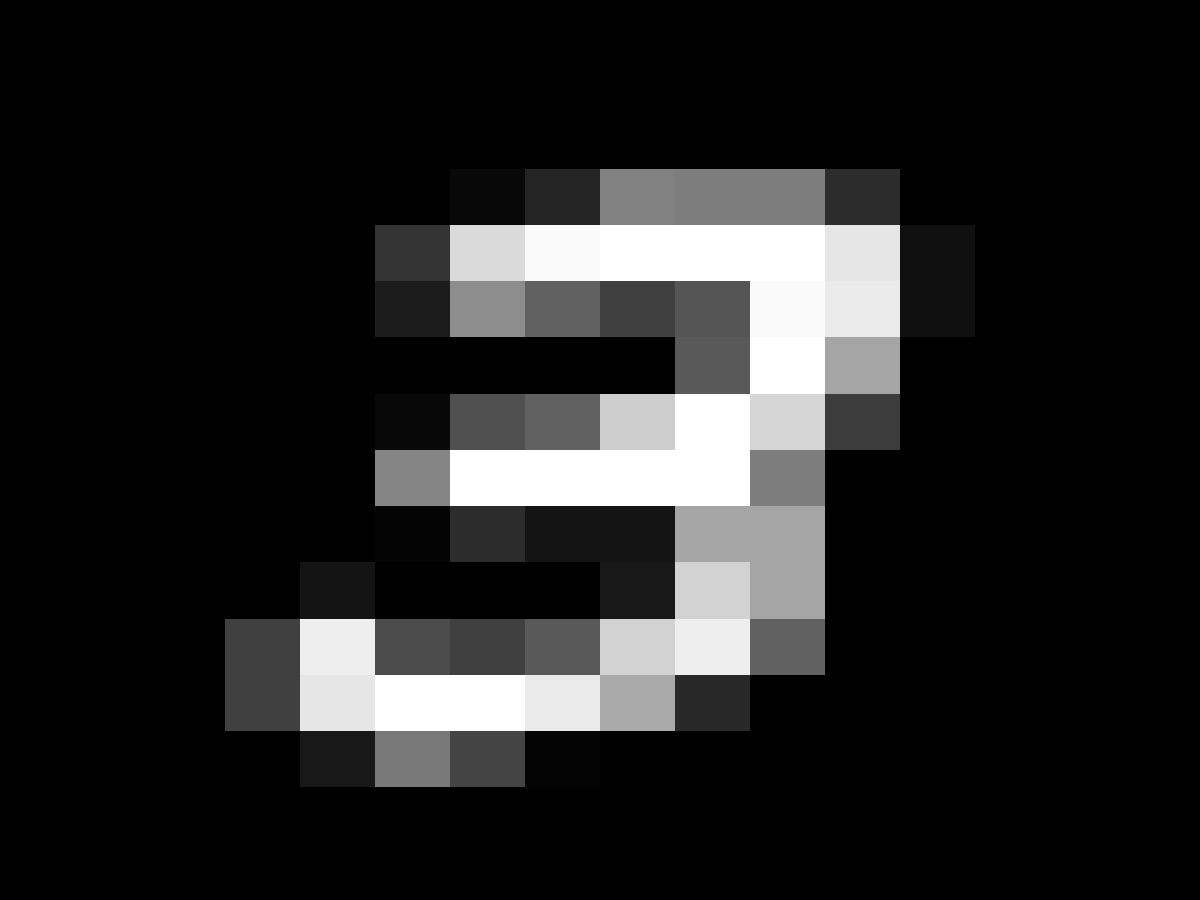}
\includegraphics[width=0.09\linewidth]
{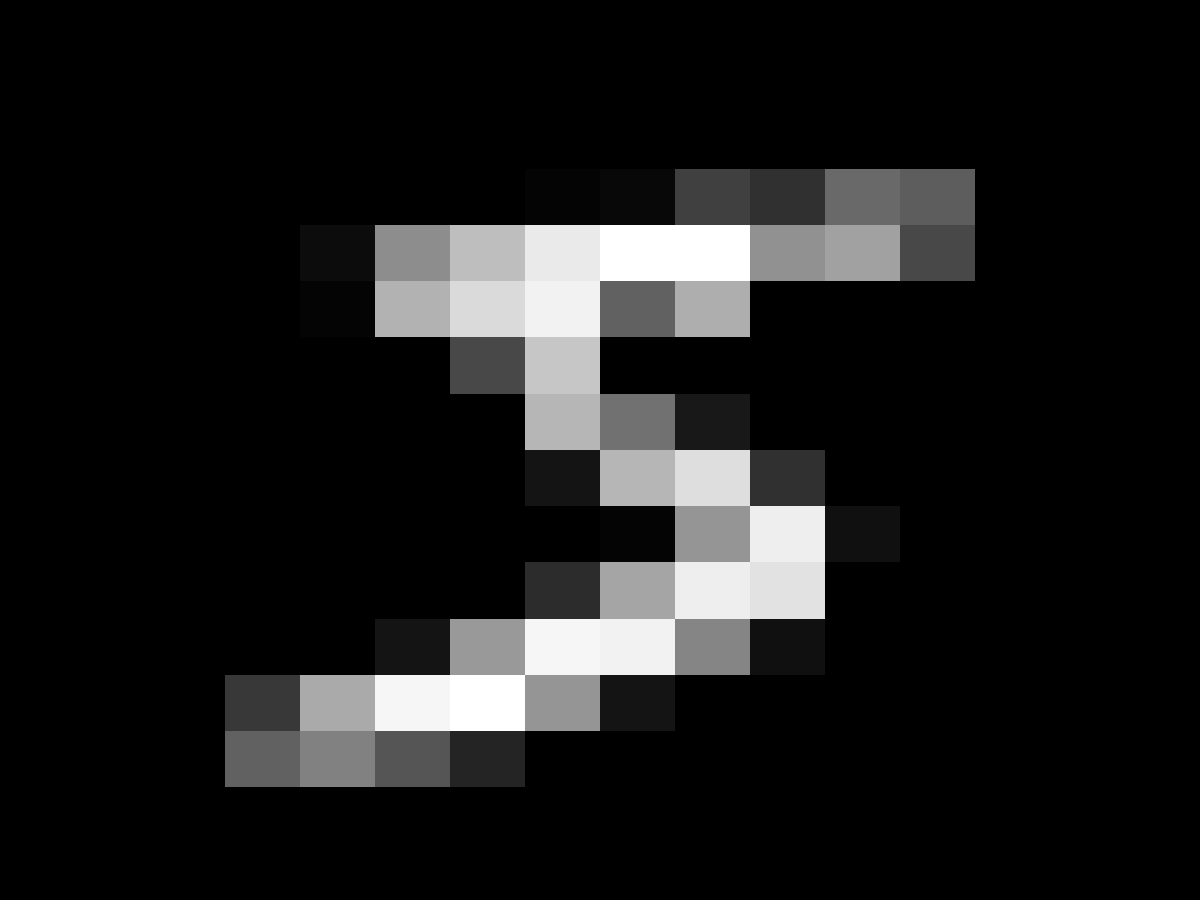}
\includegraphics[width=0.09\linewidth]
{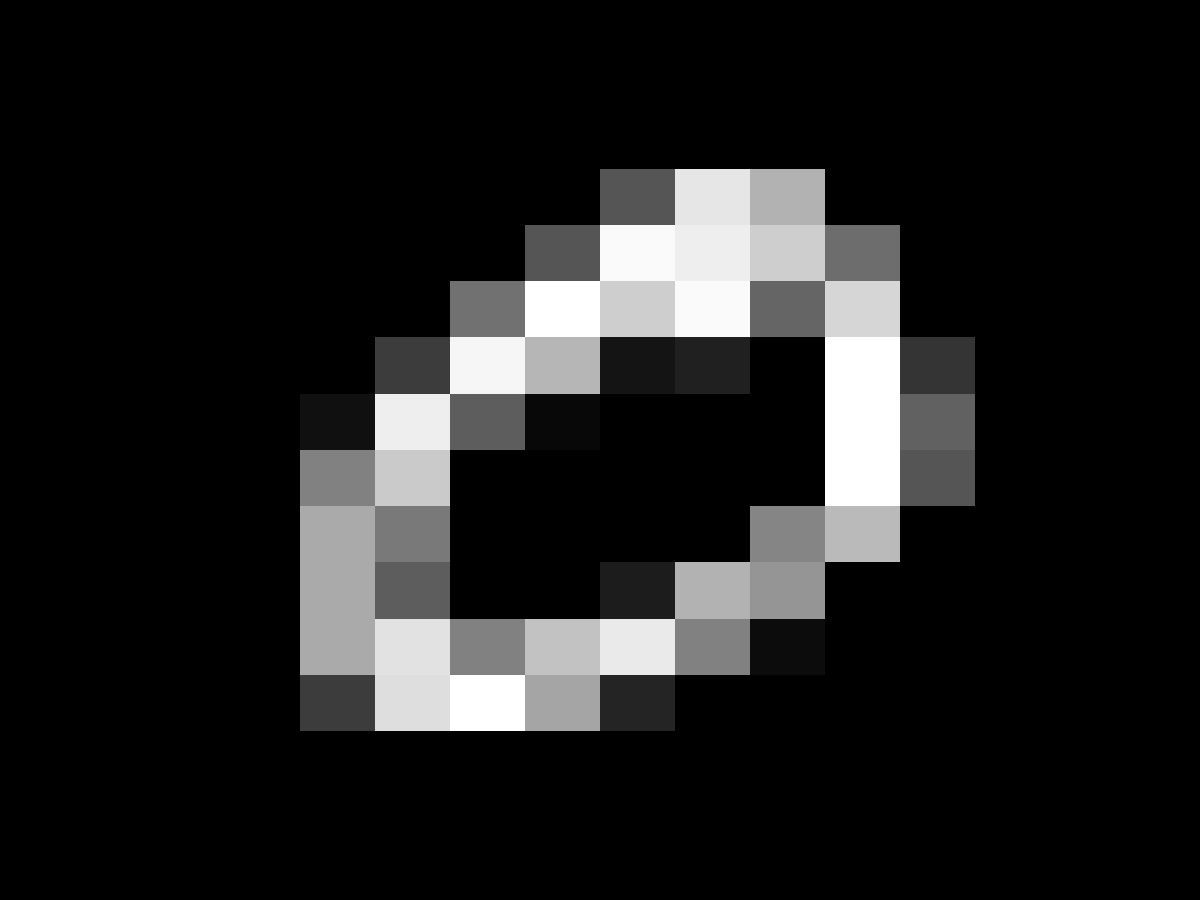}
\hspace{0.6cm}
\includegraphics[width=0.09\linewidth]
{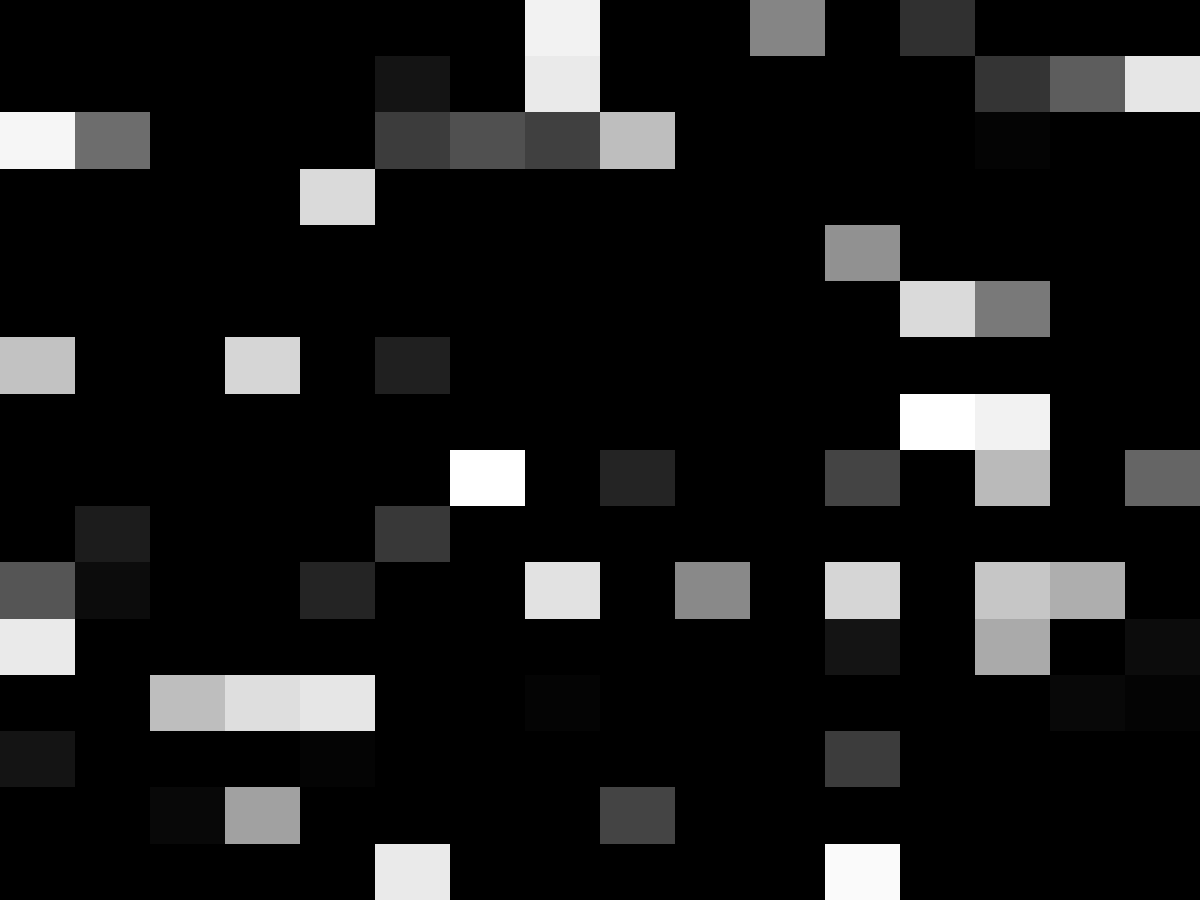}
\includegraphics[width=0.09\linewidth]
{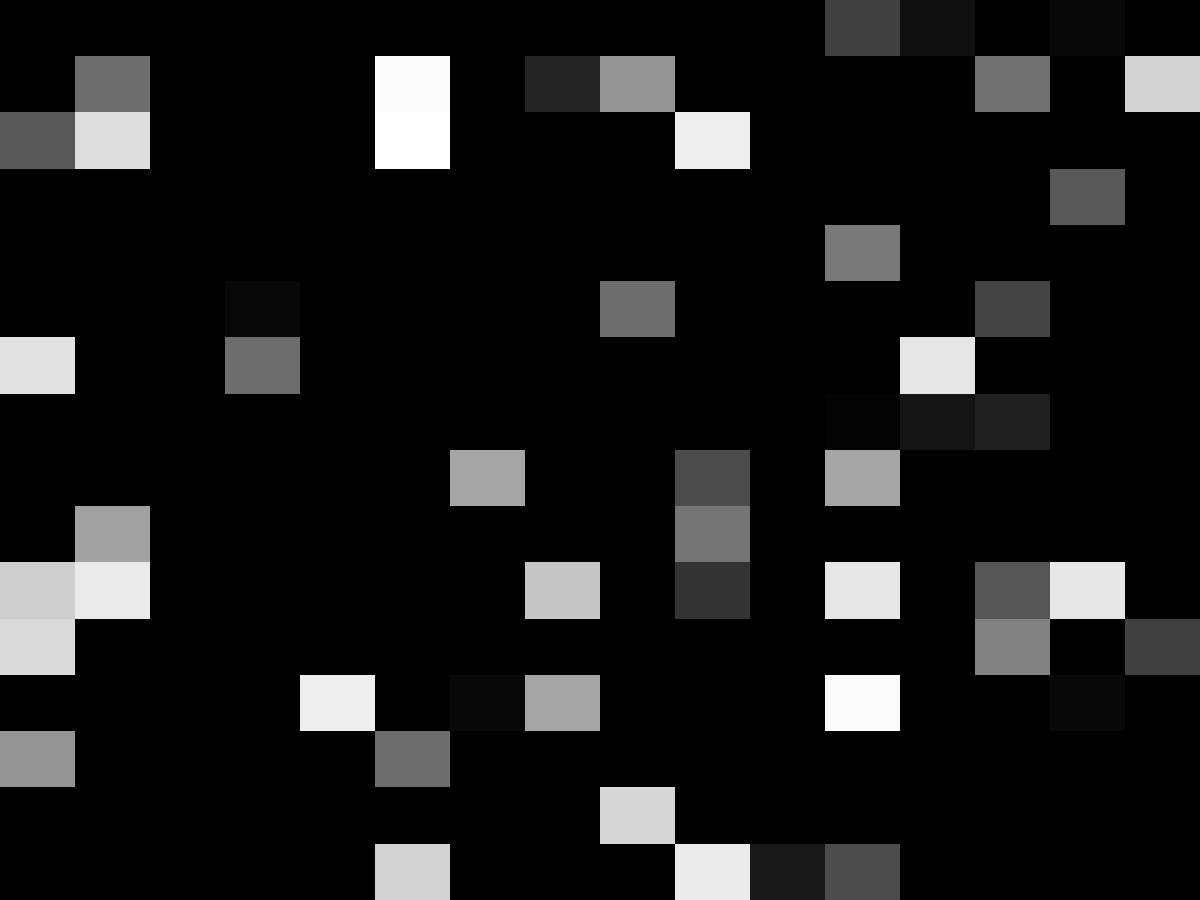}
\includegraphics[width=0.09\linewidth]
{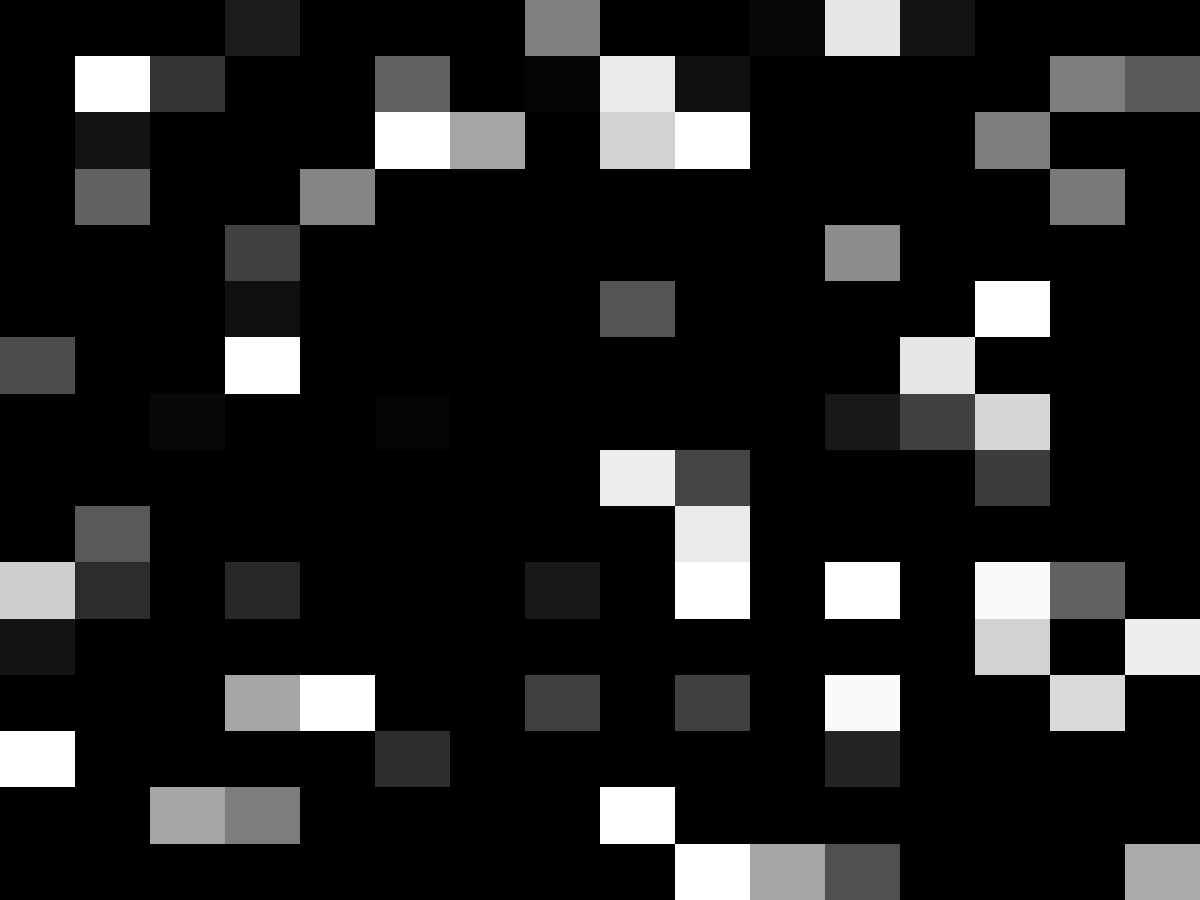}
\includegraphics[width=0.09\linewidth]
{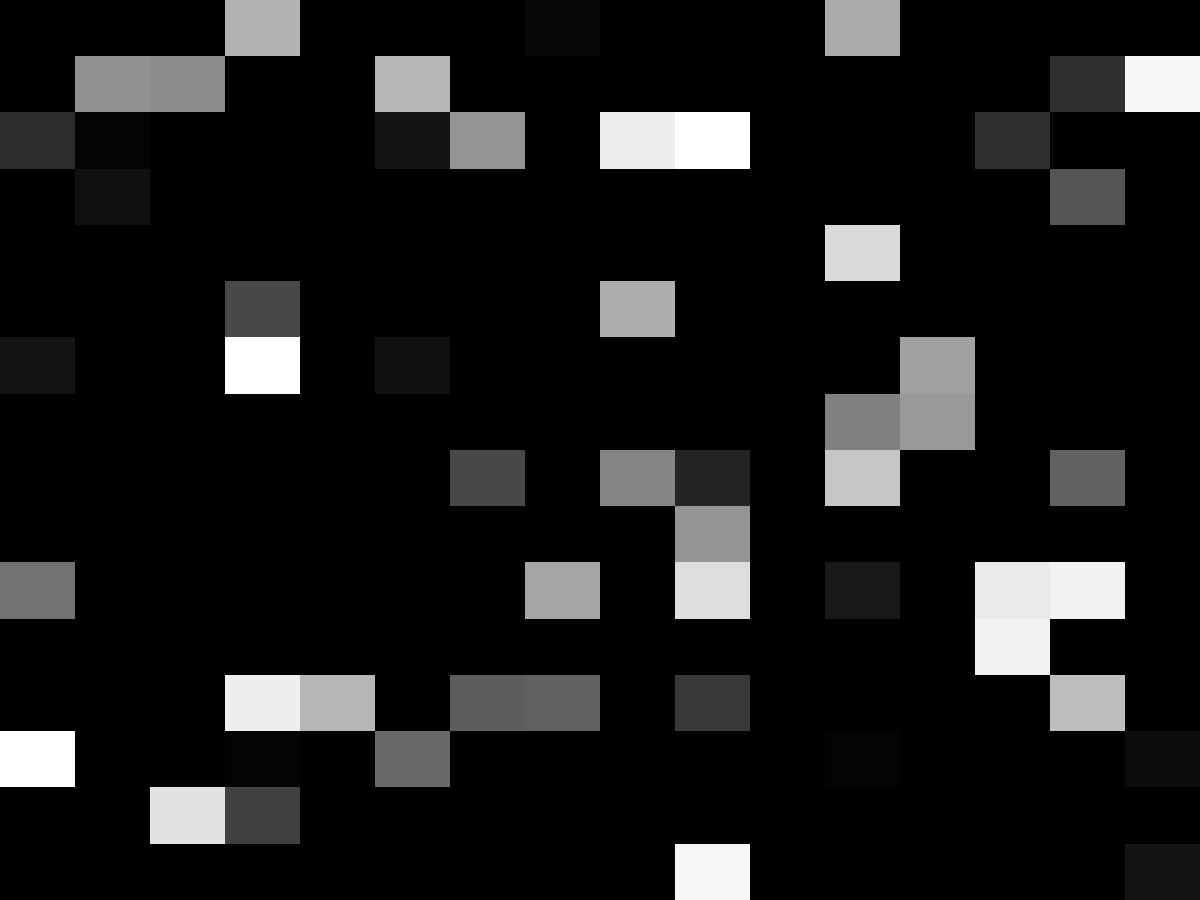}
\includegraphics[width=0.09\linewidth]
{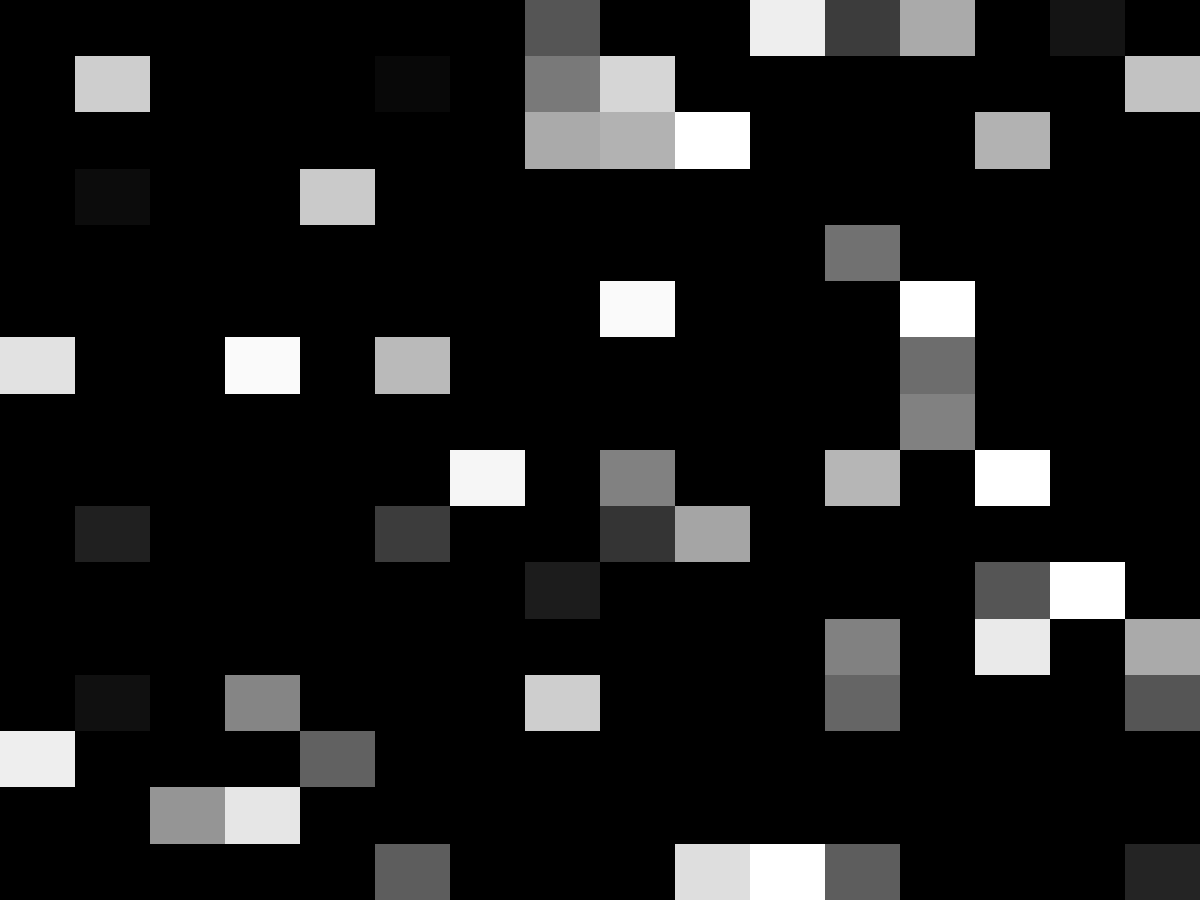}\vspace{-0.25cm}
\vspace{0.15cm}
\end{subfigure}
\caption{ 
\small MNIST images (left) and images after random pixel 
permutations (right).}
\label{fig:samplemnist}
\end{center}
\end{figure}

MNIST is a data basis with $6\times 10^4$ hand-written digit images of
size $d \leq 2^{10}$. There are $10$ classes (one per digit) with
$5\times 10^4$ images for
training and $10^4$ for testing. Examples of MNIST images 
are shown in Figure \ref{fig:samplemnist}. To test the classification
performances of a Haar scattering when the geometry is unknown, we scramble
all image pixels with the same unknown random permutations, as shown in 
Figure \ref{fig:samplemnist}.

\begin{figure}[h]
\begin{center}

\begin{subfigure}{0.8\linewidth}
\begin{center}
\begin{tabular}{c|c}
\hline
CNN (Supervised) \cite{LeCun} & 0.53 \\
Sparse Coding (Unsupervised) \cite{Labusch}& 0.59\\
Gabor Scattering \cite{Joan} & \textbf {0.43} \\
Structured Haar Scattering & 0.59\\
\hline
\end{tabular}
\caption{ 
\small Known Geometry}
\label{table:mnist1}
\end{center}
\end{subfigure}\\

\begin{subfigure}{0.6\linewidth}
\begin{center}
\begin{tabular}{c|c}
\hline
Maxout MLP + dropout \cite{Goodfellow} & 0.94 \\
Deep convex net. \cite{Yu} & 0.83 \\
DBM + dropout \cite{Hinton} & \textbf{0.79}\\
Structured Haar Scattering & 0.90\\
\hline
\end{tabular}
\caption{ 
\small Unknown Geometry}
\label{table:mnist2}
\end{center}
\end{subfigure}

\caption{\small 
Percentage of errors for the classification of MNIST images, obtained by different algorithms.}
\label{table:mnist}
\end{center}

\end{figure}

When the image geometry is known, i.e. using non-scrambled images, 
the best MNIST classification results without data augmentation
are given in Table \ref{table:mnist1}. Deep convolution
networks with supervised learning reach an error of $0.53\%$ \cite{LeCun},
and unsupervised learning with sparse coding have a slightly larger
error of $0.59\%$  \cite{Labusch}. 
A wavelet scattering computed with iterated Gabor wavelet transforms
yields an error of $0.46\%$ \cite{Joan}.

For a known image grid geometry,
we compute a structured Haar scattering by pairing neighbor image
pixels. It builds hierachical square subsets $V_{j,n}$ 
illustrated in Figure \ref{fig:2}(c).
The invariance scale is $2^J = 2^{6}$, which corresponds to blocks of
$8 \times 8$ pixels. 
Random shift and rotations of these pairing define $T=64$ different
Haar scattering transforms. The supervised classifier of
Section \ref{supsec} applied to this structured Haar scattering 
yields an error of $0.59\%$. 

MNIST digit classification is a 
relatively simple problem where the main source of variability are due to
deformations of hand-written image digits. In this case, supervised
convolution networks, sparse coding, Gabor wavelet scattering and 
orthogonal Haar scattering have nearly 
the same classification performances. The fact that a Haar
scattering is only based on additions and subtractions does not affect
its efficiency.

\begin{figure}[h!]
\centering
\includegraphics[width= 0.5 \textwidth ]{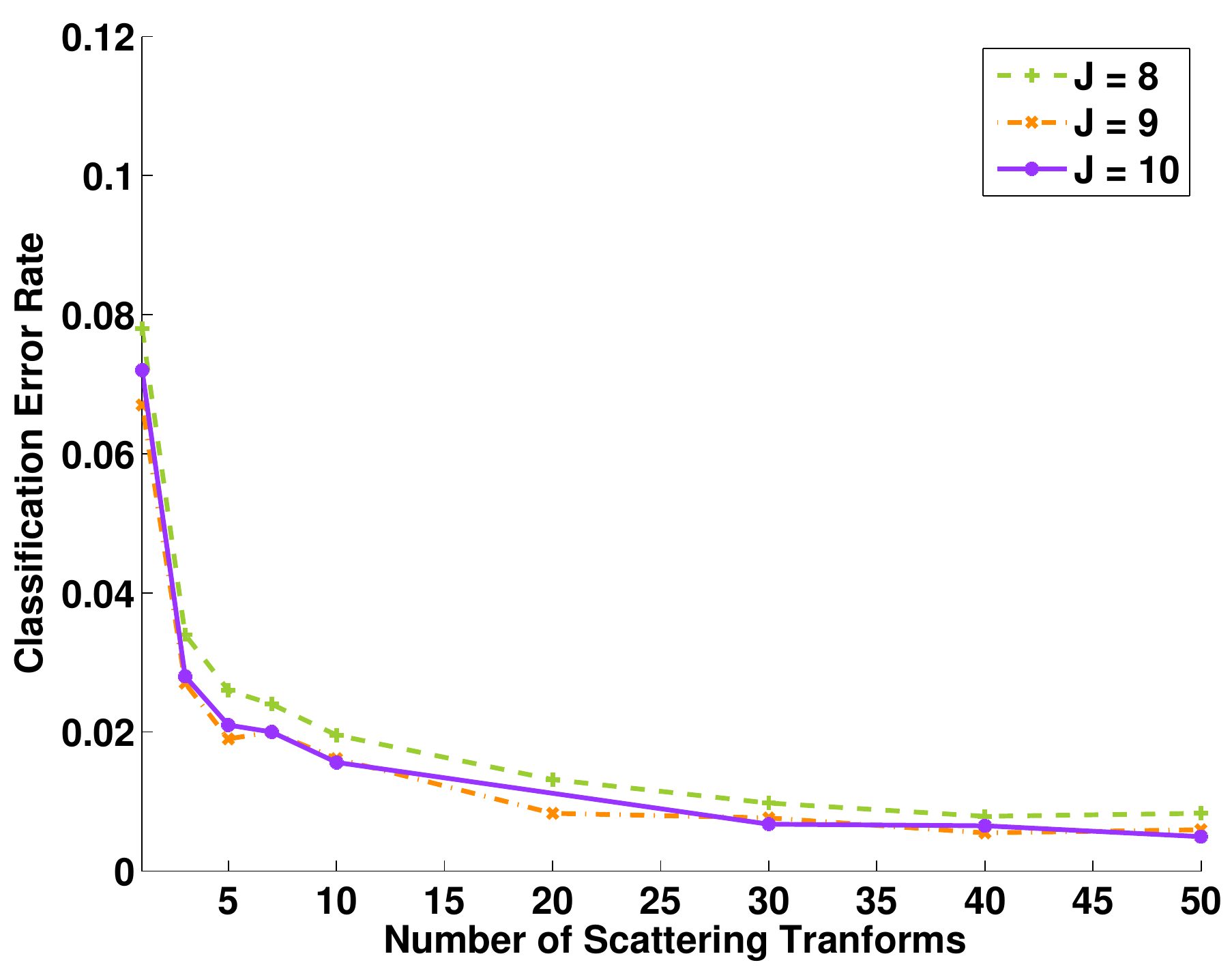}
\caption{ \small Unsupervised Haar scattering classification error for MNIST,
as a function of the number $T$ of scattering transforms, 
for networks of depth $J=8,9,10$.}
\label{fig:acc_ntree}
\end{figure}

For scrambled images, the connectivity of image pixels is unknown and needs
to be learned from data. Table \ref{table:mnist2} gives the classification
results of different learning algorithms. 
The smallest error of $0.79\%$ is
obtained with a Deep Belief  optimized with a supervised backpropagation.
Unsupervised learning of $T=50$ structured Haar scattering followed by a
feature selection and a supervised SVM classifier produces an error of $0.90\%$.
Figure \ref{fig:acc_ntree} gives
the classification error rate as a function of $T$,
for different values of maximum scale $J$. The error rates decrease slowly for $T>10$, and do not improve beyond $T=50$, which is much smaller than $2^J$.

The unsupervised learning computes connected dyadic partitions $V_{j,n}$ from scrambled images by optimizing an $\bf l^1$ norm. At scales $1 \leq 2^j \leq 2^3$, 
$100\%$ of these partitions 
are connected in the original image grid,
which proves that the geometry is well estimated at these scales. 
This is only evaluated on meaningful pixels which do not remain zero on all training images. 
For $j = 4$ and $j=5$ the percentages of connected partitions 
are respectively $85\%$ and $67\%$. The percentage of connected partitions 
decreases because long range correlations are weaker. 

A free orthogonal Haar scattering does not impose any condition on pairings.
It produces a minimum error of $1\%$ for $T=20$ Haar scattering transforms,
computed up to the depth $J=7$. This error rate is higher because the
supplement of freedom in the pairing choice increases the variance of the
estimation.

\subsection{CIFAR-10 images}\label{cifar10}

CIFAR-10 is a data basis of tiny color images of $32\times32$ pixels.
It includes $10$ classes, such as ``dogs'', ``cars'', ``ships''
with a total of $5 \times 10^4$ training examples and $10^4$ testing examples. 
There are much more intra-class variabilities than in MNIST digit images, 
as shown by Figure \ref{fig:samplecifar10}.
The $3$ color bands are represented with $Y,U,V$ channels, and scattering 
coefficients are computed independently in each channel. 

\begin{figure}[t]
\begin{center}
\includegraphics[width=0.1\linewidth, height=0.1\linewidth]
{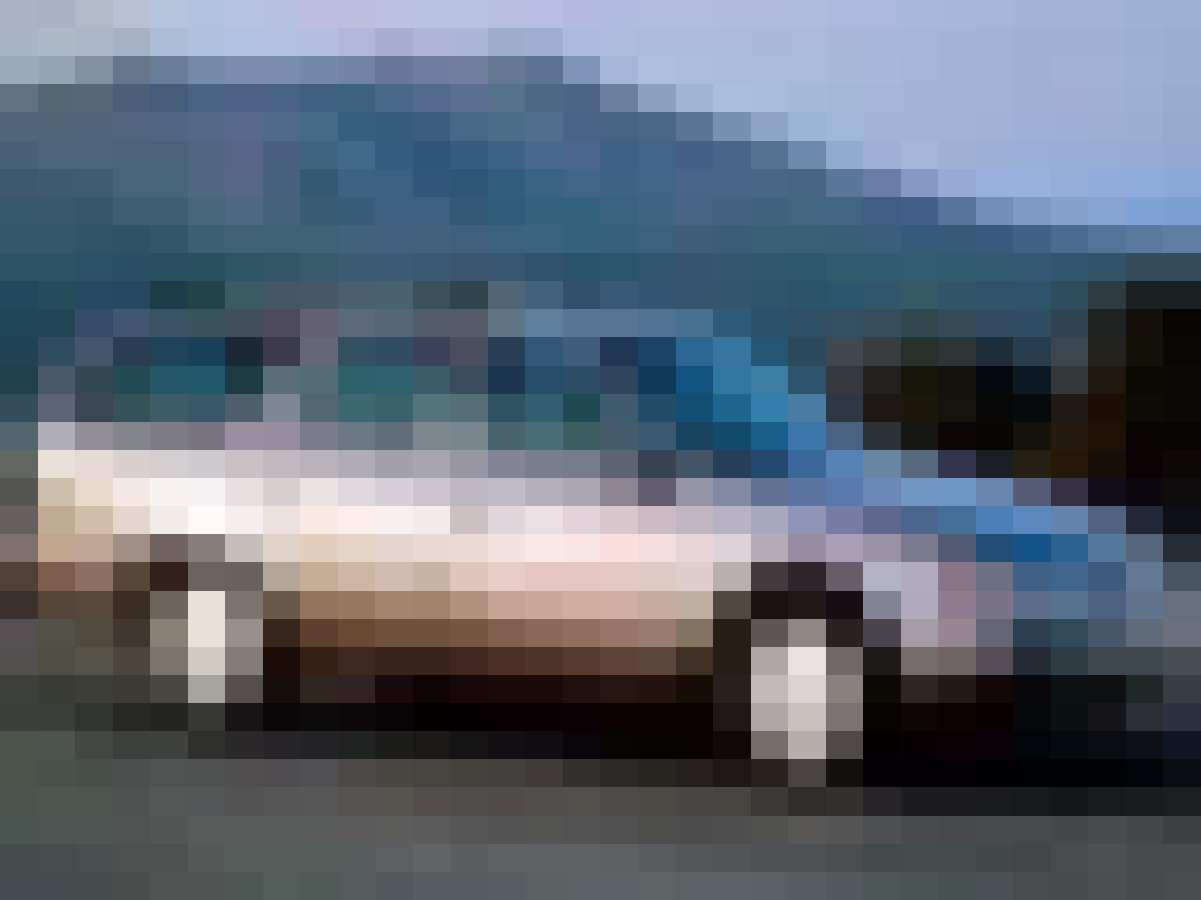}
\includegraphics[width=0.1\linewidth, height=0.1\linewidth]
{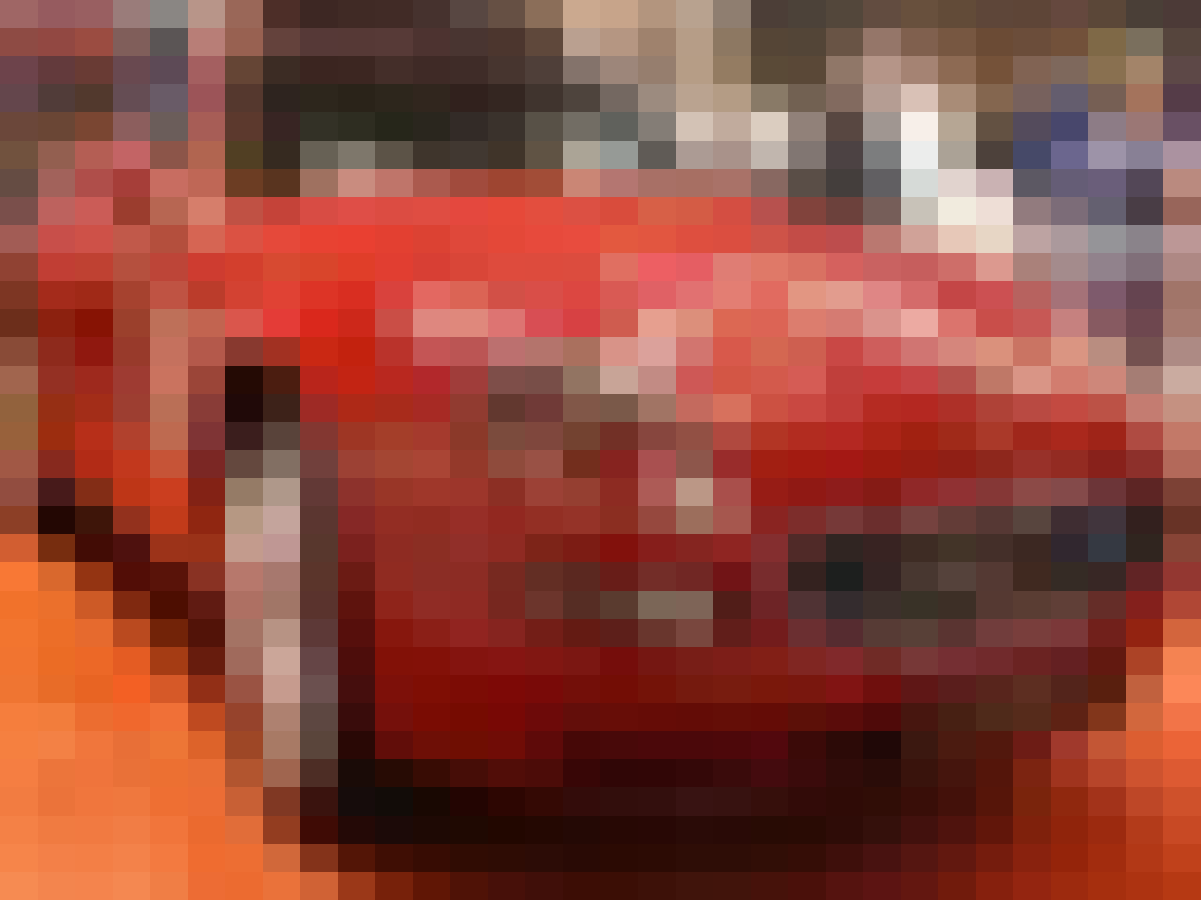}
\includegraphics[width=0.1\linewidth, height=0.1\linewidth]
{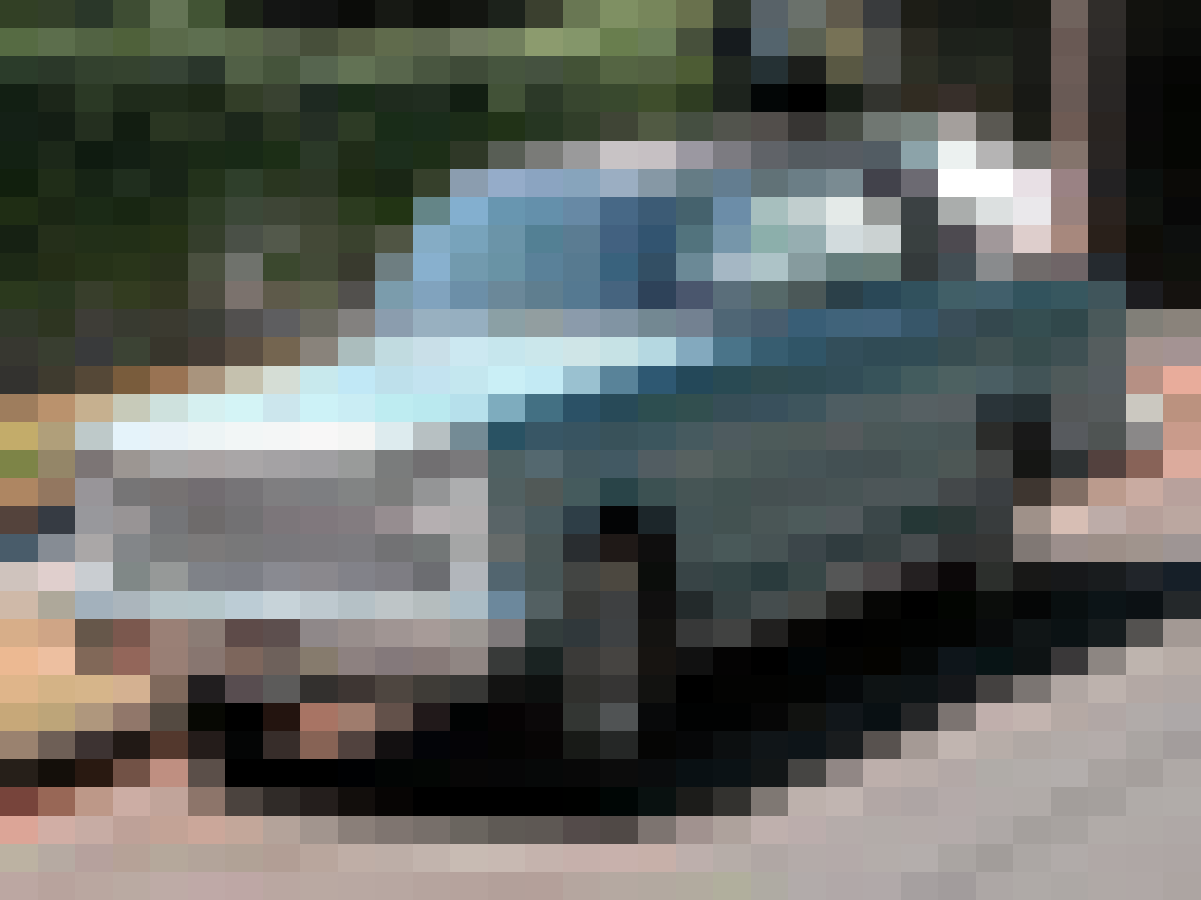}
\hspace{0.3cm}
\includegraphics[width=0.1\linewidth, height=0.1\linewidth]
{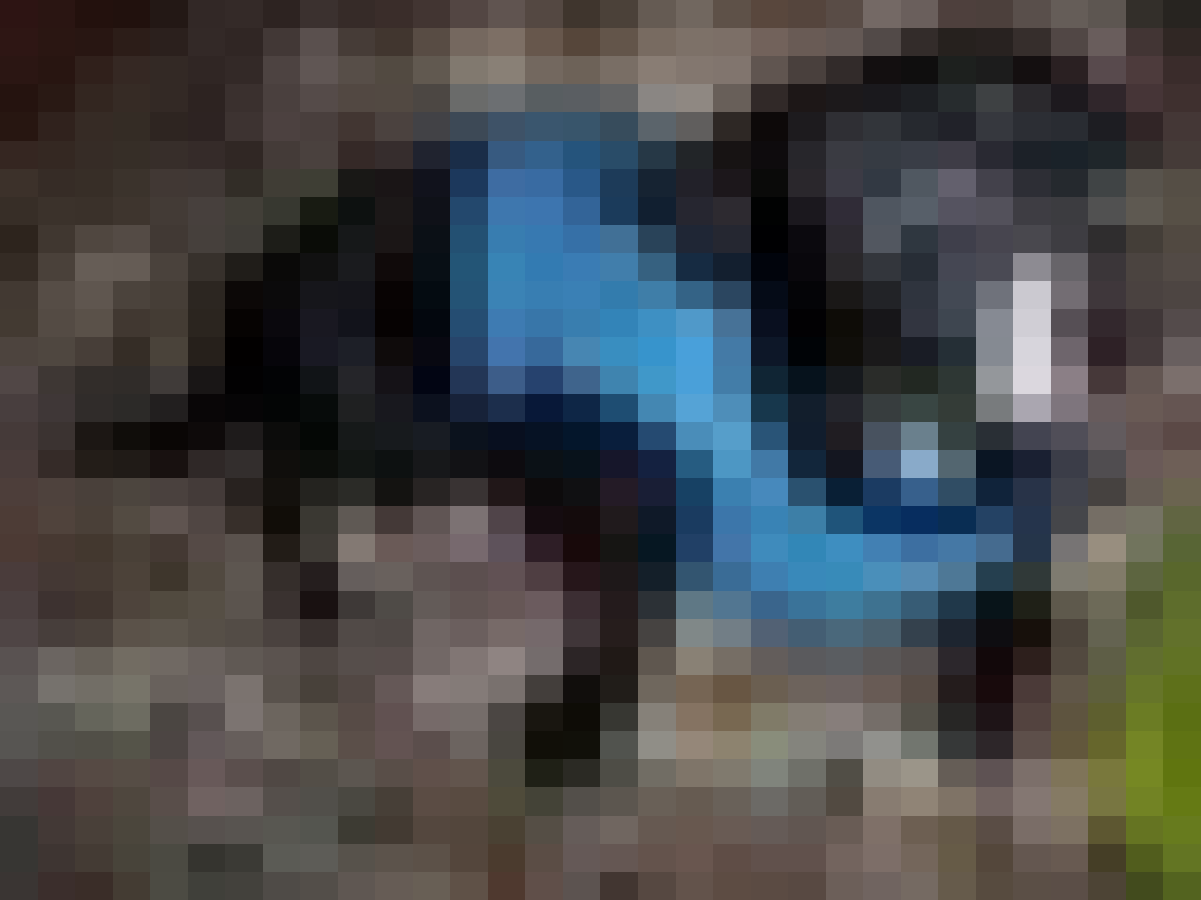}
\includegraphics[width=0.1\linewidth, height=0.1\linewidth]
{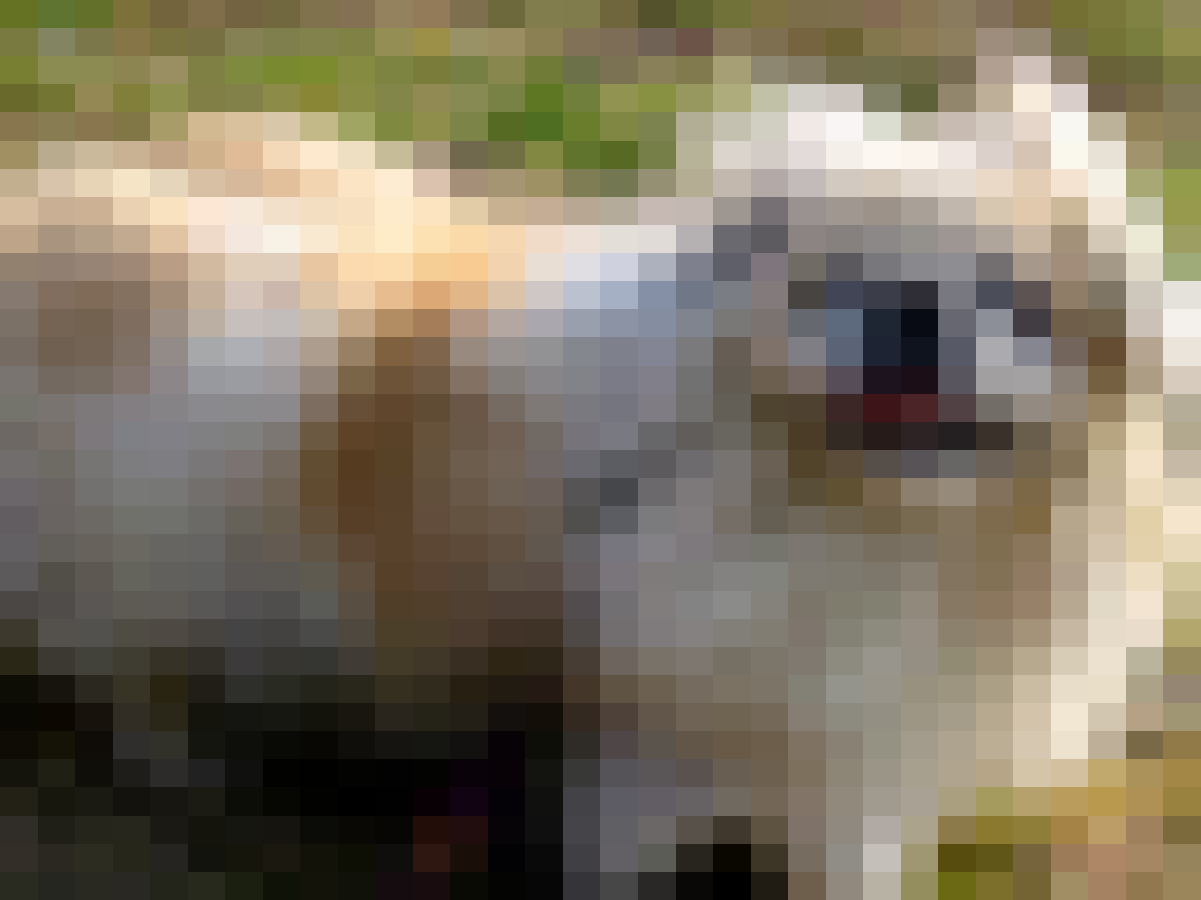}
\includegraphics[width=0.1\linewidth, height=0.1\linewidth]
{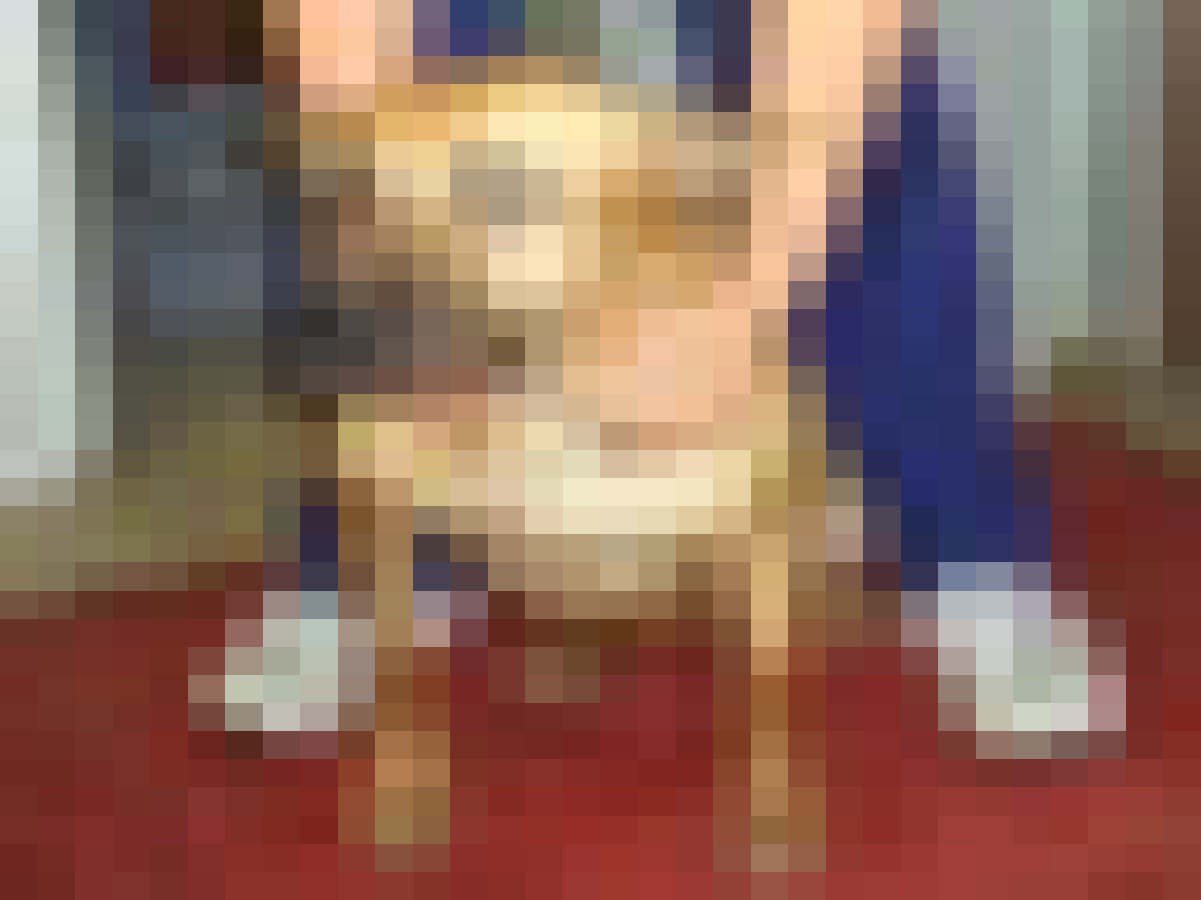}
\hspace{0.3cm}
\includegraphics[width=0.1\linewidth, height=0.1\linewidth]
{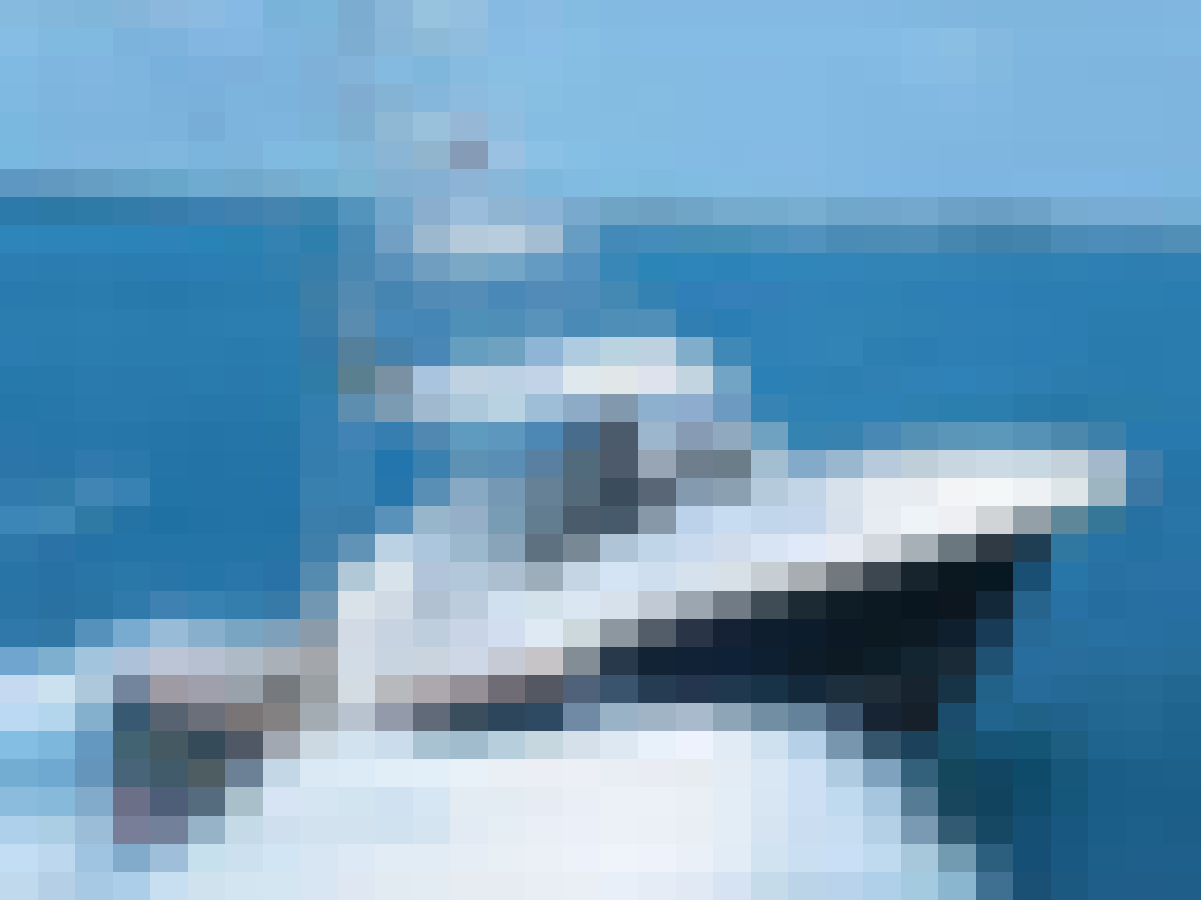}
\includegraphics[width=0.1\linewidth, height=0.1\linewidth]
{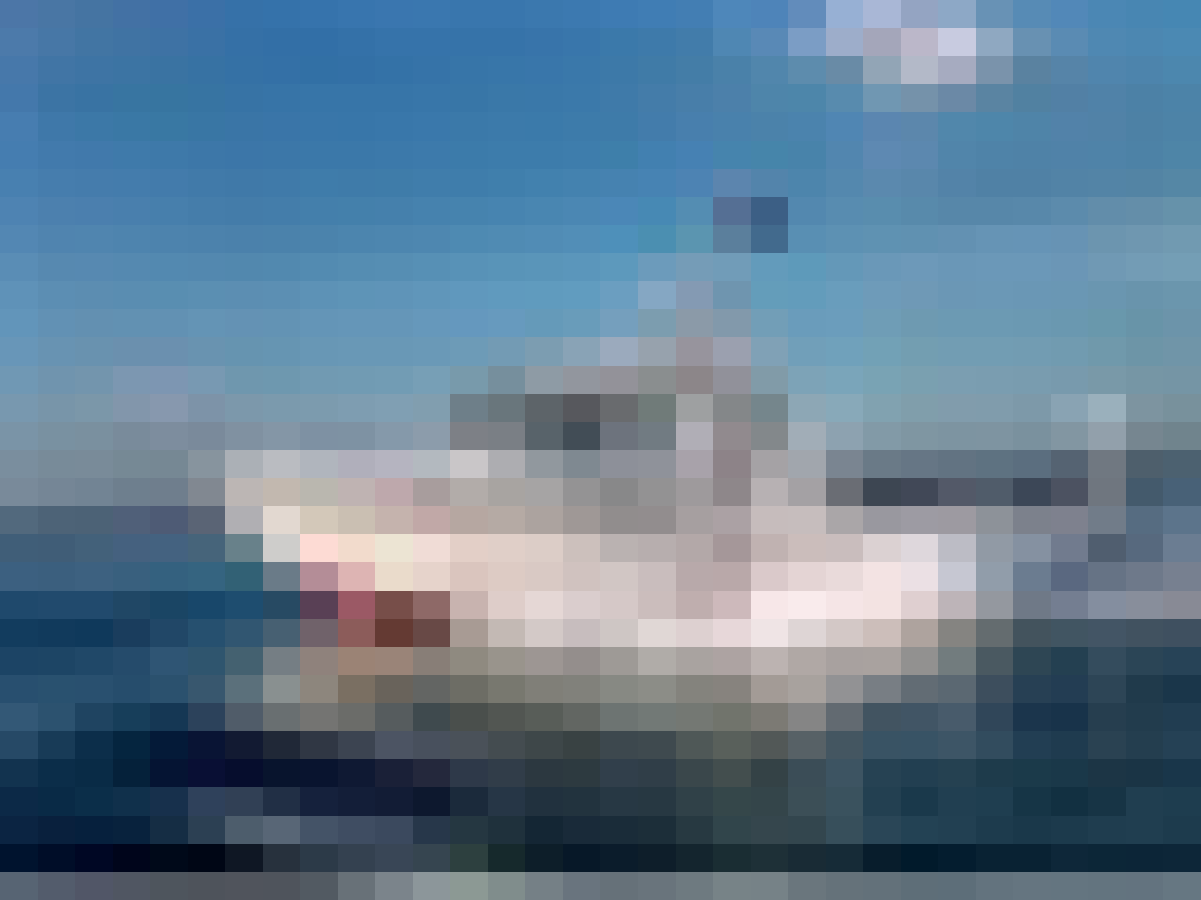}
\includegraphics[width=0.1\linewidth, height=0.1\linewidth]
{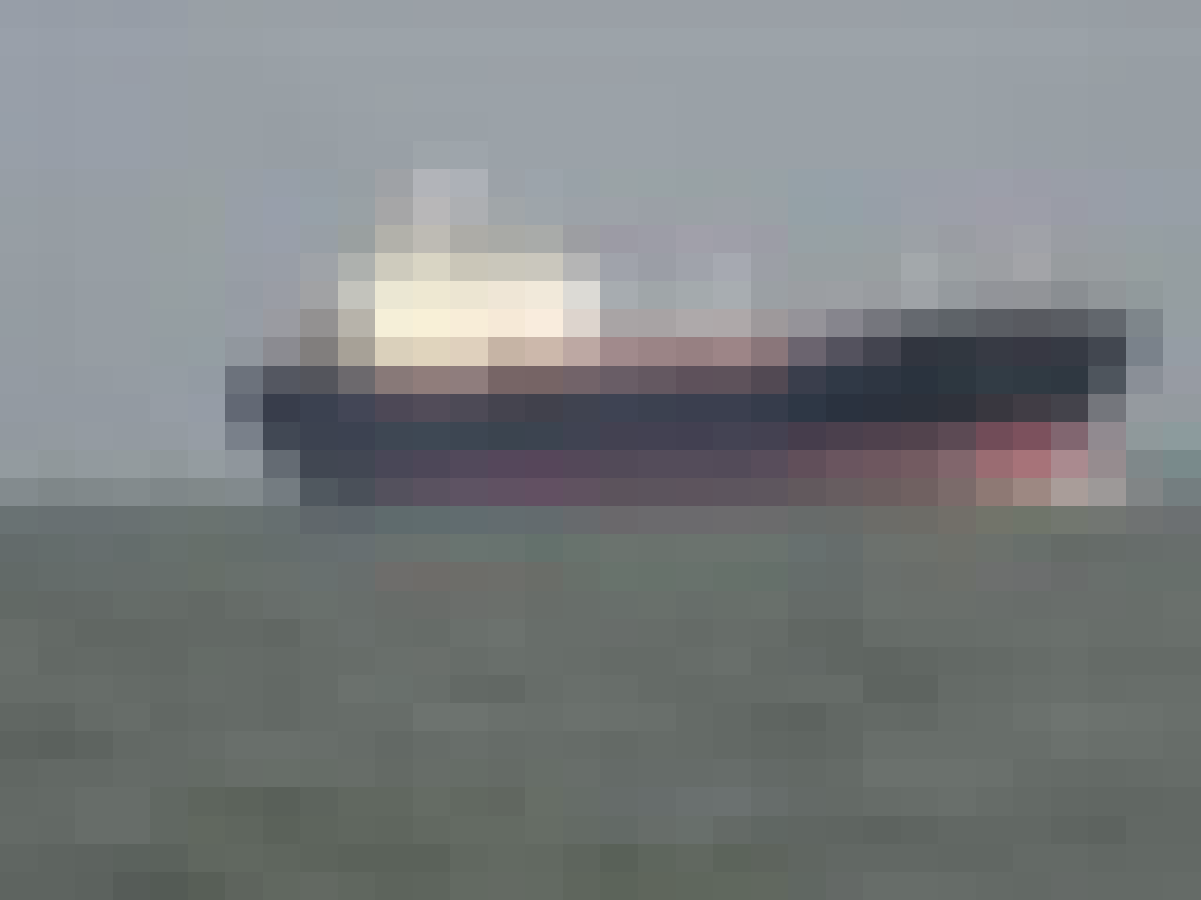}
\vspace{0.15cm}
\caption{ 
\small Examples of CIFAR-10 images in the classes of ``cars'', ``dogs'' and ``boats''.}
\label{fig:samplecifar10}
\end{center}
\end{figure}

When the image geometry is known,
a structured Haar scattering is computed by pairing neighbor image
pixels. The best performance is obtained at the scale $2^J = 2^6$ which is below the maximum scale $d = 2^{10}$. Similarly to MNIST, we compute $T = 64$  connected dyadic partitions for randomly translated and rotated grids. 
After dimension reduction, the classification error is $21.3\%$. 
This error is above state
of the art results 
of unsupervised learning algorithms by about $20\%$, but it involves no learning.
A minimum error rate of $16.9\%$ is
obtained by Receptive Field Learning \cite{Jia}.
The Haar scattering error is also above the $17.8\%$ error obtaiend by a 
roto-translation invariant wavelet scattering network \cite{Oyallon},
which computes wavelet transforms
along translation and rotation parameters. 
Supervised deep convolution
networks provide an important improvement over all unsupervised techniques
and reach an error of $9.8\%$. The study of these supervised networks is however beyound
the scope of this paper. 
Results are summarized in Table \ref{table:cifar1}. 

When the image grid geometry is unknown, because of random scrambling, 
Table \ref{table:cifar1} summarizes results with different algorithms.
For unsupervised learning with structured Haar scattering,
the minimum classification error is reached at the scale $2^J = 2^7$,
which maintains some localization information on scattering coefficients.
With $T = 10$ connected dyadic partitions, the error is $27.3\%$. 
Table \ref{table:cifar2} shows that 
it is $10\%$ below previously reported results on this data basis.

Nearly $100\%$ of the dyadic paritions $V_{j,n}$ computed from scrambled images are connected in the original image grid, for $1 \leq j \leq 4$, which
shows that the multiscale geometry is well estimated at these fine scales.  
For $j=5, 6$ and $7$, the proportions of connected partitions are $98\%$, $93\%$ and $83\%$ respectively. As for MNIST images, the connectivity
estimation becomes less precise at large scales. 
Similarly to  MNIST, a free Haar scattering yields a higher classification error of 29.2\%, with $T=20$ scattering transforms up to layer $J=6$.

\begin{figure}[h]
\begin{center}
\begin{subfigure}{0.6\linewidth}
\begin{center}
\begin{tabular}{c|c}
\hline
CNN (Supervised state-of-the-art) \cite{Lee} & \textbf{9.8} \\
RFL (Unsupervised state-of-the-art) \cite{Jia}& 16.9\\
Roto-Translation Scattering \cite{Oyallon} & 17.8 \\
Structured Haar Scattering & 21.3\\
\hline
\end{tabular}
\vspace{0.15cm}
\caption{ 
\small Known Geometry}
\label{table:cifar1}
\end{center}
\end{subfigure}
\hspace{2.0cm}
\begin{subfigure}{0.6\linewidth}
\begin{center}
\begin{tabular}{c|c}
\hline
Fastfood \cite{Le} & 37.6 \\
Fastfood FFT \cite{Le} & 36.9 \\
Random Kitchen Sinks \cite{Le} & 37.6 \\
Structured Haar Scattering & \textbf{27.3}\\
\hline
\end{tabular}
\caption{ 
\small Unknown Geometry}
\label{table:cifar2}
\vspace{0.15cm}
\end{center}
\end{subfigure}
\caption{\small 
Percentage of errors for the classification of CIFAR-10 images, obtained by different algorithms.}
\label{table:cifar}
\end{center}
\end{figure}

\subsection{CIFAR-100 images}\label{cifar100}

CIFAR-100 also contains tiny color images of the same size as CIFAR-10 images. It has 100 classes containing 600 images each, of which 500 are training images and 100 are for testing. Our tests on CIFAR-100 follows the same procedures as in Section \ref{cifar10}. The 3 color channels are processed independently. 

When the image grid geometry is known, the results of a structured Haar scattering are summarized in Table \ref{table:cifar100}. The best performance is obtained with the same parameter combination as in CIFAR-10, which is $T=64$ and $2^J = 2^6$. After dimension reduction, the classification error is 47.4\%. As in CIFAR-10, this error is about $20 \%$ 
larger than 
state of the art unsupervised methods, such as a Nonnegative OMP ($39.2\%$)\cite{Lin}.
A roto-translation wavelet scattering has an error of $43.7\%$. 
Deep convolution networks with
supervised training produce again a lower error of $34.6\%$.

For scrambled images of unknown geometry, with $T=10$ 
transforms and a depth $J = 7$, a structured Haar scattering has an error of 
$52.7\%$. 
A free Haar orthogonal scattering has a higher classification error of $56.1\%$, with $T=10$ scattering transforms up to layer $J=6$. No such result is
reported with another algorithm on this data basis.

On all tested image databases, structured Haar scattering has a consistent 7\%-10\% performance advantage over `free' Haar scattering, as shown by 
Table \ref{table:l1l2}.
For orthogonal Haar scattering, all reported errors were calculated
with an unsupervised learning which minimizes the $\bf l^1$ norm 
(\ref{l1energy}) of scattering coefficients, layer per play.
As expected, Table \ref{table:l1l2} shows that minimizing
a mixed $\bf l^1$ and $\bf l^2$
norm  (\ref{l1energy0}) yields nearly the same results on all data bases.  

\begin{figure}[h]
\begin{center}
\begin{tabular}{c|c}
\hline
CNN (Supervised state-of-the-art) \cite{Lee} & \textbf{34.6} \\
NOMP (Unsupervised state-of-the-art) \cite{Lin}& 39.2\\
Gabor Scattering \cite{Oyallon} & 43.7 \\
Structured Haar Scattering & 47.4\\
\hline
\end{tabular}
\vspace{0.15cm}
\caption{ 
\small Percentage of errors for the classification of CIFAR-100 images with known geometry, obtained by different algorithms.}
\label{table:cifar100}
\end{center}
\end{figure}

\begin{table}[h]
\begin{center}
\begin{tabular}{c|c|c|c|c}
\hline 
& Structured, $\bf l^1$ & Structured, $\bf l^1$/$\bf l^2$ &Free, $\bf l^1$  & Free, $\bf l^2$/$\bf l^1$\\
\hline
MNIST& \textbf{0.91} & 0.95 & 1.09   & 1.02 \\
\hline
CIFAR-10& 28.8  & \textbf{27.3}  & 29.2   & 29.3 \\
\hline
CIFAR-100&  \textbf{52.5} & 53.1  & 56.3 & 56.1 \\
\hline
\end{tabular}
\end{center}\vspace{-0.25cm}
\caption{\small Percentage of errors for the classification of MNIST, CIFAR-10 and CIFAR-100 images with a structured or a free Haar scattering, for unsupervised
computed by minimizing a mixed $\bf l^1$/$\bf l^2$ norm or an $\bf l^1$ norm.}
\label{table:l1l2}
\end{table}

\subsection{Images on a graph over a sphere}
\label{3D}

A data basis of irregularly sampled images on a sphere 
is provided in \cite{Joan2}. 
It is constructed by projecting the MNIST image digits on 
$d=4096$ points randomly sampled on the 3D sphere, and by randomly rotating these
images on the sphere. The random rotation is either uniformly distributed on the sphere 
or restricted with a smaller variance (small rotations) \cite{Joan2}.
The digit `9' is removed from the data set because it can not be
distinguished from a `6' after rotation. Examples sphere digits are shown in Figure \ref{fig:spheremnist}. This geometry of points on the sphere can be described
by a graph which connects points having a sufficiently small distance on the 
sphere.

\begin{figure}[t]
\begin{center}
\includegraphics[width=0.24\linewidth]
{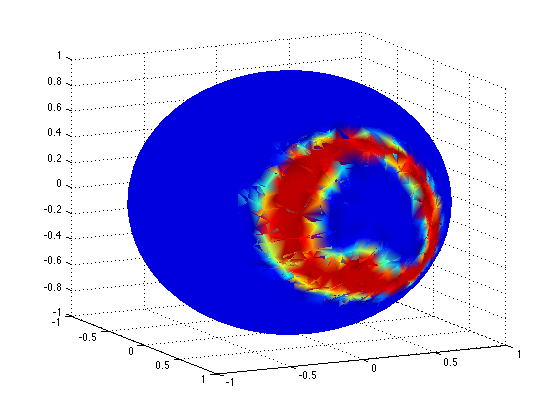}
\includegraphics[width=0.24\linewidth]
{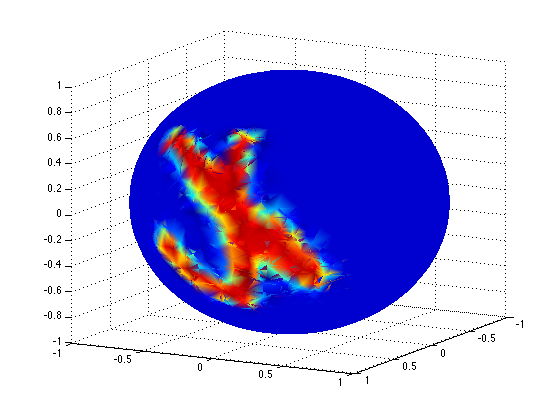}
\includegraphics[width=0.24\linewidth]
{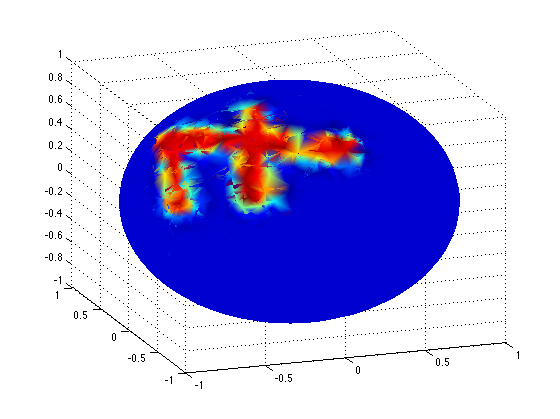}
\includegraphics[width=0.24\linewidth]
{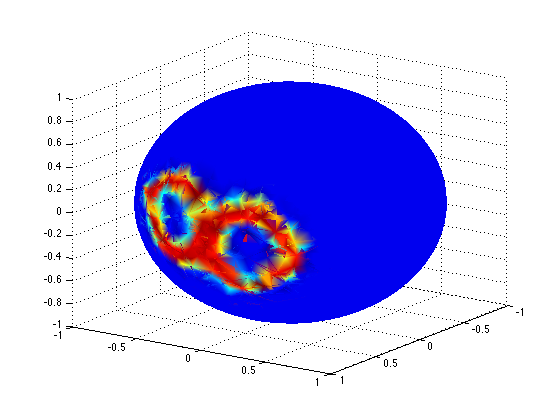}
\caption{\small Images of digits mapped on a sphere.}
\label{fig:spheremnist}
\end{center}
\end{figure}

The classification algorithms introduced in \cite{Joan2} 
take advantage of the known distribution of points on the sphere,
with a representation based on the graph Laplacian. 
Table \ref{table:smnist} gives the results reported in \cite{Joan2},
with a fully connected neural network, and with a spectral graph Laplacian network.

\begin{table}[h]
\resizebox{\textwidth}{!}{
\begin{tabular}{c|c|c|c|c|c}
\hline 
&Nearest& Fully  & Spectral  & Structured Haar & Free Haar\\
&Neighbors & Connect.  & Net.\cite{Joan2} & Scattering & Scattering\\
\hline
Small rotations& 19 & 5.6 & 6 & 2.2 & \textbf{1.6} \\
\hline
Large rotations& 80 & 52 & 50 & \textbf{47.7} & 55.8 \\
\hline
\end{tabular}}
\caption{\small Percentage of errors for the classification of MNIST images
rotated and sampled on a sphere
\cite{Joan2}, with a nearest neighbor classifier, a
fully connected two layer neural network, 
a spectral network \cite{Joan2}, and 
an unsupervised Haar scattering.}
\label{table:smnist}
\end{table}

As opposed to these algorithms, the unsupervised
structured Haar scattering algorithm 
does not use this geometric information and learns the graph information
by pairing. Computations are
performed on a scrambled set of signal values. 
Haar scattering transforms are calculated up to the maximum scale 
$2^J = d = 2^{12}$.  
A total of $T = 10$  connected dyadic partitions are estimated by unsupervised learning, and the classification is performed from $M = 10^3$ selected
coefficients. Although the graph geometry is unknown, the structured Haar scattering reduces the error rate both for small and large 3D random rotations. 

In this case a free orthogonal Haar scattering 
has a smaller error rate than a structured Haar scattering 
for small rotations, but a larger error for large rotations. It illustrates the tradeoff between the structural bias and the feature variance in the choice of the algorithms. For small rotation, the variability within classes is smaller and a free scattering can take advantage of more
degrees of freedom. For large rotations, the variance is too large
and dominates the problem. 

Two points of the sphere of radius $1$ are considered to be
connected if their geodesic distance is smaller than $0.1$. With this convention,
over the 4096 points,
each point has on average $8$ connected neighbors. 
The unsupervised Haar learning performs a hierachical
pairing of points on the sphere. For small and large rotations, the percentage of 
connected sets $V_{j,n}$ remains above $90\%$
for $1 \leq j \leq 4$. This is computed over $70\%$ of the points
points having a nonneglegible energy.
It shows that the multiscale geometry on the sphere is well estimated by hierachical
pairings.

\section*{Acknowledgment}

This work was supported by the ERC grant InvariantClass 320959.


\bibliographystyle{plain}

%


%
%
%


\clearpage

\appendix

\numberwithin{equation}{section}
\numberwithin{theorem}{section}

\section{Proof of Theorem \ref{prop:order-m-coeff}}
\label{app:1}

\begin{proof}[Proof of Theorem \ref{prop:order-m-coeff}]
We derive from the definition of a scattering transform in equations (3,4) in the text that
\begin{equation*}
\begin{split}
S_{j+1} x(n, 2 q) 
  &= S_j x(\pi_j(2n), q) + S_j x(\pi_j(2n+1), q) 
   = \lb \overline{S}_j x(\cdot, q ), 1_{V_{j+1}, n}\rb,  \\
S_{j+1} x(n, 2 q+ 1)
 & = | S_j x(\pi_j(2n), q) - S_j x(\pi_j(2n+1), q) | 
   =  | \lb \overline{S}_j x(\cdot, q ), \psi_{j+1, n} \rb|. 
\end{split}
\end{equation*}
where $V_{j+1,n} = V_{j, \pi_j(2n)} \cup  V_{j, \pi_j(2n+1)}$.
Define $\kappa = 2^{-j}q = \sum_{k=1}^m 2^{-j_k}$.
Observe that
\begin{equation*}
2^{j_{m+1}} ( \kappa + 2^{-j_{m+1}})
 =  2^{j_{m+1}} \kappa + 1
 = 2 (  2^{j_{m+1}-1} \kappa ) + 1,
\end{equation*}
thus $S_{j_{m+1}} x (n,  2^{j_{m+1}} ( \kappa + 2^{-j_{m+1}})  )$ is calculated from the coefficients $ S_{j_{m+1} -1} x ( n ,   2^{j_{m+1}-1} \kappa   )  $  of the previous layer with
\begin{equation}
\label{eqn:jm+1}
S_{j_{m+1}} x (n,   2^{j_{m+1}} ( \kappa + 2^{-j_{m+1}})   ) 
 = |\lb \overline{S}_{j_{m+1} -1} x (\cdot,   2^{j_{m+1}-1} \kappa   ), \psi_{j_{m+1}, n} \rb|.
\end{equation}
Since $2^{j+1}\kappa = 2 \cdot 2^{j}\kappa$,  the coefficient 
$S_{j_{m+1}-1} x (n,   2^{j_{m+1}-1} \kappa    )$
is calculated from $S_{j_{m}} x (n, 2^{j_m} \kappa  )$ by $(j_{m+1}-1-j_m)$ times additions, and thus
\begin{equation}
\label{eqn:fromjmtojm+1-1}
S_{j_{m+1} -1} x (n,   2^{j_{m+1}-1} \kappa   )
 = \lb \overline{S}_{j_m} x(\cdot,  2^{j_m} \kappa   ), 1_{V_{j_{m+1}-1, n}} \rb.
\end{equation}
Combining equations (\ref{eqn:fromjmtojm+1-1}) and (\ref{eqn:jm+1}) gives
\begin{equation}
\label{eqn:jm+1b}
S_{j_{m+1}} x (n,   2^{j_{m+1}} ( \kappa + 2^{-j_{m+1}}) ) 
 = |\lb \overline{S}_{j_m} x(\cdot,   2^{j_m} \kappa  ), \psi_{j_{m+1}, n} \rb|.
\end{equation}
We go from the depth $j_{m+1}$ to the depth $j \geq j_{m+1}$ by computing
\[
 S_j x (n,  2^{j} ( \kappa + 2^{-j_{m+1}})   ) 
= \lb \overline {S}_{j_{m+1}} x (\cdot,   2^{j_{m+1}} ( \kappa + 2^{-j_{m+1}})  ) , 1_{V_{j, n}} \rb .
\]
Together with (\ref{eqn:jm+1b}) it proves the equation (\ref{propsdfnsd}) of
the proposition. The summation over $p, \, V_{j_{m+1}, p}\subset V_{j, n} $ comes from the inner product $\lb 1_{V_{j_{m+1}, p}}, 1_{V_{j,n}}  \rb$.
This also proves that $\kappa + 2^{-j_{m+1}}$ is the index
of a coefficient of order $m+1$.
\end{proof}

\section{Proof of Theorem \protect \ref{thm:n>dlogd}}
\label{app2}


The theorem is proved by analyzing the concentration of the  objective function 
around its expected value as the sample number $N$ increases. 
We firstly introduce the Pisier and Maurey's version of the Gaussian concentration inequality for Lipschitz functions.
\begin{proposition}[gaussian concentration for Lipschitz function \cite{pisier1985, maurey1991}] 
\label{prop:gaussian-concentration-Lipschitz}
Let $z_{1},...z_{m}$ be i.i.d $N(0,1)$ random variabls, and $f=f(z_{1},\cdots,z_{m})$
a 1-Lipschitz function, then there exists $c_{0}>0$ so that 
\[
\Pr[ f -\E f > t] <\exp\{-c_{0}t^{2} \}~
\mbox{and}~ 
\Pr[ f - \E f< - t]<\exp\{-c_{0}t^{2} \}
,\quad\forall t>0.
\]
\end{proposition}
In the above proposition, the constant $c_0 =\frac{2}{\pi^{2}}$ according to \cite{pisier1985} and $1/4$ in \cite{maurey1991}.

To prove the theorem, recall that the pairing problem is
computed by minimizing the $\bf l^1$ norm (\ref{msdifnsdf2}) which 
up to a normalization amounts to compute:
\begin{equation}
\pi^{*}=\arg\min_{ \pi \in \Pi_{d}} F(\pi)~\mbox{with}~
F(\pi ) = \frac{1}{N} \sum_{i=1}^{N } \sum_{(u,v)\in \pi}|x_{i}(u)-x_{i}(v)|,
\label{eq:pairing-problem-layer1}
\end{equation}
where $\pi$ is a pairing of $d$ elements and we denote by $\Pi_d$ the set of all possible such pairings.

The following lemma proves that
$F( \pi )$ is a Lipschitz function of independent gaussian random variables,
with a Lipschitz constant equal to $\|\Sigma_d\|_{op}^{1/2}$, where
$\|\Sigma_d\|_{op}$ is the operator norm of the covariance. 
We prove it on the normalized function $f = N^{1/2} d^{-1/2} F$. 

\begin{lemma}\label{lemma:F-Lipschitz}
 Let $x_{i}=\Sigma_{d}^{1/2} z_{i}$ with $z_{i}=(z_{i}(1),...,z_{i}(d))^{T}\sim \N (0,I_{d})$ i.i.d. Given any pairing $\pi \in\Pi_{d}$, define 
\[
f (\{(z_{i}(v)\}_{1\le i\le N, 1\le v\le d})
=\frac{1}{\sqrt{d N}}\sum_{i=1}^{N}\sum_{(u,v)\in \pi} |x_{i}(u)-x_{i}(v)|,
\]
then $f$ is a Lipschitz function with constant $\sqrt{ \| \Sigma_d\|_{op} }$, which does not depend on $\pi$. 
\end{lemma}

\begin{proof} 
With slight abuse of notation, denote by $v = \pi (u)$ if two nodes $u$ and $v$ are paired by $\pi$, 
then we have 
\begin{eqnarray*}
\frac{\partial f}{\partial z_{i}(v')} 
 & = & \frac{1}{\sqrt{d N}}\sum_{(u,v)\in \pi} \text{Sgn}(x_{i}(u)-x_{i}(v))\frac{\partial}{\partial z_{i}(v')}(x_{i}(u)-x_{i}(v))\\
 & = & \frac{1}{\sqrt{d N}}\sum_{u=1}^{d}\text{Sgn}(x_{i}(u)-x_{i}(\pi (u)))\frac{\partial}{\partial z_{i}(v')}x_{i}(u)\\
 & = & \frac{1}{\sqrt{d N}}\sum_{u=1}^{d}\text{Sgn}(x_{i}(u)-x_{i}(\pi (u)))(\Sigma_{d}^{1/2})_{u,v'}\\
 & = & \frac{1}{\sqrt{d N}}(\Sigma_{d}^{1/2}S_{i})(v'),
\end{eqnarray*}
where $S_{i}:=(\text{Sgn}(x_{i}(u)-x_{i}( \pi (u))))_{u=1}^{d}$ is a vector of length $d$ whose entries are $\pm1$. Then
\[
\|\Sigma_{d}^{1/2}S_{i}\| \le \sqrt{ \| \Sigma_d\|_{op} d},
\]
and it follows that 
\begin{eqnarray*}
\|\nabla_{z}f\|^{2} 
 & = & \sum_{i=1}^{N} \sum_{v'=1}^{d}\left|\frac{\partial f}{\partial z_{i}(v')}\right|^{2}\\
 & = & \sum_{i=1}^{N} \frac{1}{dN}\|\Sigma_{d}^{1/2}S_{i}\|^{2}
 \le \| \Sigma_d\|_{op}.
\end{eqnarray*}
\end{proof}

Observe that  the eigenvalues of $\Sigma_d$ are the discrete Fourier transform coefficients of the periodic correlation function $\rho(u)$
\[
\hat{\rho}(k)
=\sum_{j=0}^{d-1}\rho(j)\exp\{-i2\pi\frac{jk}{d}\}
=\sum_{j=0}^{d-1}\rho(j)\cos(2\pi\frac{jk}{d}),
\quad k=0,...,d-1.
\]
Observe that  $\sum_{u=1}^{d}S_{i}(u)=0$ for each $i$, that is, $S_{i}$ is orthogonal to the eigenvector of $\hat{\rho}(0)$. So the Lipschitz constant $\sqrt{ \|\Sigma_d\|_{op}} = \sqrt{\max_{k} |\hat{\rho}(k)| }$ can be slightly improved to be $ \sqrt{\max_{k > 0} |\hat{\rho}(k)| }$.

Let us now prove the claim of Theorem \ref{thm:n>dlogd}. Since the pairing has a probability larger than $1 - \epsilon$  to be connected if $\Pr [\pi^{*}\notin\Pi_{d}^{(0)}] < \epsilon$, we need to show that under the inequality (\ref{theoansdf}) the probability $\Pr[ \pi^* \notin\Pi_{d}^{(0)} ]$ is less than $\epsilon$.
Let us denote
\begin{equation}
\label{alphadef}
\alpha_{u}=\sqrt{\frac{2}{\pi}\cdot2(1-\rho(u))},
~\mbox{and}~
\bar{\alpha}_2 = \min_{ 2 \le u \le d/2 } \alpha_u,
\end{equation}
and define 
\begin{equation}
\label{Crho}
C_{\rho}=\frac{c_{0}}{\|\Sigma_d\|_{op}}\left(\frac{1}{2}(\bar{\alpha}_{2}-\alpha_{1})\right)^{2}.
\end{equation}
Then Eqn. (\ref{theoansdf}) can be rewritten as 
\begin{equation}
C_{\rho}\frac{N}{d} > 3\log d - \log \epsilon.
\label{theoansdf-2}
\end{equation}

As a result of Proposition \ref{prop:gaussian-concentration-Lipschitz} and Lemma \ref{lemma:F-Lipschitz}, if $C=\frac{c_{0}}{\|\Sigma_d\|_{op}} \cdot N/d$ then $\forall \pi \in\Pi_{d}$,
\[
\Pr[ F(\pi ) - \E F(\pi )>\delta]<\exp\{-C\delta^{2}\}
~\mbox{and}~
\, \Pr[ F(\pi) -  \E F(\pi)  <- \delta]<\exp\{-C\delta^{2}\}
,\quad\forall\delta>0.
\]
Observe that 
\[
\Pi_{d}=\bigcup_{m=0}^{d/2}\Pi_{d}^{(m)}
\]
where $\Pi_{d}^{(m)}$ are the set of pairings which have $m$ non-neighbor pairs. $\Pi_{d}^{(0)}$ is the set of pairings which only pair connected nodes in the graph, and for the ring graph $\Pi_{d}^{(0)}=\{ \pi_{0}^{(0)}, \pi_{1}^{(0)} \}$ the two of which interlace. 
For any $ \pi \in \Pi_{d}^{(m)}$, suppose that there are $m_{l}$ pairs in $\pi $ so that the distance between the two paired nodes is $l$, $m_{1}=d/2-m$, $m_{2}+\cdots+m_{d/2}=m$.

Recalling the definition of $\alpha_k$ in Eq. (\ref{alphadef}), we verify that
\[
\E F( \pi )
= \alpha_{1}(\frac{d}{2}-m)+\alpha_{2}m_{2}+\cdots+
   \alpha_{d/2}m_{d/2}
\ge \alpha_{1}(\frac{d}{2}-m)+\bar{\alpha}_{2}m
\]
when $m\ge1$, and 
\[
\E F( \pi_{0}^{(0)}) =\E F( \pi_{1}^{(0)}) =\alpha_{1}\frac{d}{2}.
\]
Thus when $m\ge1$, 
\[
\E F( \pi )- \E F( \pi_{0}^{(0)})  \ge ( \bar{\alpha}_{2}-\alpha_{1})m, \quad\forall \pi \in\Pi_{d}^{(m)}.
\]
Define 
\[
\delta_{m}= \frac{1}{2}( \bar{\alpha}_{2}-\alpha_{1})m,\quad m=1,...,d/2,
\]
and we have that 
\begin{eqnarray}
\Pr [ \pi^{*} \notin\Pi_{d}^{(0)}] 
 & = & \Pr[\exists \pi \in\bigcup_{m=1}^{d/2}\Pi_{d}^{(m)},\,   F( \pi )<\min\{F(\pi_{0}^{(0)}),F( \pi_{1}^{(0)})\}]\nonumber \\
 & \le & \Pr [F(\pi_{0}^{(0)})>\E F(\pi_{0}^{(0)})+\delta_{1}]\nonumber \\
 &  &  +\Pr [F(\pi_{0}^{(0)})<\E F(\pi_{0}^{(0)})+\delta_{1},\,
	\exists \pi \in \bigcup_{m=1}^{d/2}\Pi_{d}^{(m)},\, F( \pi )<F(\pi_{0}^{(0)})]\nonumber \\
 & \le & \Pr [F(\pi_{0}^{(0)})> \E F(\pi_{0}^{(0)})+\delta_{1}] \nonumber \\
 &  & + \Pr[\exists \pi \in\bigcup_{m=1}^{d/2}\Pi_{d}^{(m)},\, F(\pi )<\E F(\pi )-\delta_{m}],\,
        (\text{by that }(\bar{\alpha}_{2}-\alpha_{1})m-\delta_{1}\ge\delta_{m})\nonumber \\
 & \le & \exp\{-C\delta_{1}^{2}\}+\sum_{m=1}^{d/2}|\Pi_{d}^{(m)}|\exp\{-C\delta_{m}^{2}\}\nonumber \\
 & = & \exp\{-C_{\rho}\frac{N}{d}\}
         +\sum_{m=1}^{d/2}|\Pi_{d}^{(m)}|\exp\{-C_{\rho}\frac{N}{d}m^{2}\},
\label{eq:prob-fail1}
\end{eqnarray}
where $C_\rho$ is as in Eq. (\ref{Crho}).

One can verify the following upper bound for the cardinal number of $\Pi_{d}^{(m)}$: 
\[
|\Pi_{d}^{(m)}|\le\frac{d^{2m}}{(2m)!}.
\]
With the crude bound $(2m)!\ge1$, the above inequality 
inserted in (\ref{eq:prob-fail1}) gives
\begin{equation}
\Pr [ \pi^{*}\notin\Pi_{d}^{(0)}]
\le\exp\{-C_{\rho}\frac{N}{d}\}+\sum_{m=1}^{d/2}d^{2m}\exp\{-C_{\rho}\frac{N}{d}m^{2}\}.
\label{eq:prob-fail2}
\end{equation}
If we keep the factor $(2m)! $, the upper bound for the summation over $m$ can be improved to be
\[ 
d^2\exp\{ -C_{\rho} N/d \}\sum_{m=1}^{d/2} ((2m)!)^{-1} 
\le c \cdot d^2\exp\{ -C_{\rho} N/d \} 
\] 
 where $c=(e-1)/2$ is an absolute constant. By applying this in 
the final bound in the theorem, the constant in front of $\log d$ is 2 instead of 3. The constant of the theorem is not tight, while the $O(d\log d)$ is believed to be the tight order as $d$ increases.

To proceed, define the function 
\[
g(x)=-C_{\rho}\frac{N}{d}\cdot x^{2}+(2\log d)\cdot x,\quad1\le x\le\frac{d}{2},
\]
and observe that $\max_{1\le x\le d/2}g(x)= g(1)$ whenever 
\[
\frac{\log d}{C_{\rho} N/d}<1,
\]
which holds as long as Eq. (\ref{theoansdf-2}) is satisfied. Thus we have
\[
\sum_{m=1}^{d/2}d^{2m}\exp\{-C_{\rho}\frac{N}{d}m^{2}\}
\le \sum_{m=1}^{d/2}d^{2}\exp\{-C_{\rho}\frac{N}{d}\}
  =  \frac{d^{3}}{2}\exp\{-C_{\rho}\frac{N}{d}\},
\]
then the inequality (\ref{eq:prob-fail2}) becomes
\[
\Pr [\pi^{*}\notin\Pi_{d}^{(0)}]
 \le\left(\frac{d^{3}}{2}+1\right)\exp\{-C_{\rho}\frac{N}{d}\}
\le\exp\{-C_{\rho}\frac{N}{d}+3\log d\}.
\]
To have $\Pr [\pi^{*}\notin\Pi_{d}^{(0)}] < \epsilon$, a sufficient condition is therefore
\begin{equation}
\label{oisdfw}
\exp\{-C_{\rho}\frac{N}{d}+3\log d\} < \epsilon,
\end{equation}
which is reduced to Eq. (\ref{theoansdf-2}) and equivalently Eq. (\ref{theoansdf}).


\end{document}